\documentclass[10pt,twocolumn,letterpaper]{article}

\usepackage{cvpr}              %

\definecolor{cvprblue}{rgb}{0.21,0.49,0.74}
\usepackage[pagebackref,breaklinks,colorlinks,allcolors=cvprblue]{hyperref}

\def\paperID{18788} %
\def\confName{CVPR}
\def\confYear{2026}

\usepackage{booktabs}
\usepackage{color, colortbl, soul}
\usepackage{multirow}
\usepackage{algorithm}
\usepackage{graphicx}
\usepackage{adjustbox}
\usepackage{algpseudocode}
\usepackage{amsmath}
\usepackage{amsthm}
\usepackage{placeins}
\usepackage{lscape}
\usepackage{tablefootnote}
\usepackage{siunitx}

\definecolor{lightgray}{gray}{0.92} %
\definecolor{darkblue}{RGB}{0, 51, 153} %
\definecolor{highlight}{RGB}{255, 230, 204} %
\definecolor{lowbetter}{RGB}{120,180,255} %
\newcounter{nTheorems}

\newtheorem{theorem}[nTheorems]{Theorem}
\newtheorem{corollary}[nTheorems]{Corollary}

\newtheorem{proposition}[nTheorems]{Proposition}

\newcounter{nDefinitions}
\newtheorem{definition}[nDefinitions]{Definition}

\newtheorem{innercustomgeneric}{\customgenericname}
\providecommand{\customgenericname}{}
\newcommand{\newcustomtheorem}[2]{%
  \newenvironment{#1}[1]
  {%
   \renewcommand\customgenericname{#2}%
   \renewcommand\theinnercustomgeneric{##1}%
   \innercustomgeneric
  }
  {\endinnercustomgeneric}
}

\newcommand{\singularsum}{C\sum^{n_{t+1}}_{i=k+1}\sigma^2_i}
\newcommand{\hegn}{h^k}
\newcommand{\htrue}{h^*}
\newcommand{\htruek}{{h^*}^k}
\newcommand{\hnew}{h_{t+1}}
\newcommand{\Frisk}{R^F}
\newcommand{\emprisk}{\hat{R}^F}
\newcommand{\Rademacher}{\mathcal{R}_{S_{t+1}}}
\newcommand{\loss}{\mathcal{\ell}^F}
\newcommand{\tnew}{t+1}
\newcommand{\Wtrue}{D^*}
\newcommand{\Wtruek}{\epsilon^* V^T_k}
\newcommand{\Wnew}{D_{t+1}}
\newcommand{\Wegn}{\epsilon V^T_k}
\newcommand{\Vk}{V^T_k}

\newcommand{\size}{S_{t+1}}

\newcustomtheorem{customthm}{Theorem}
\newcustomtheorem{customlemma}{Lemma}
\usepackage[utf8]{inputenc}
\usepackage{newunicodechar}

\title{Shared LoRA Subspaces for \textit{almost} Strict Continual Learning}

\author{Prakhar Kaushik\thanks{equal contribution}\,\,\thanks{Corresponding author: \texttt{prakhark2@gmail.com}}, Ankit Vaidya\footnotemark[1], Shravan Chaudhari, Rama Chellappa, Alan Yuille \\\\
\normalsize{Department of Computer Science}\\
\normalsize{Johns Hopkins University}\\
\normalsize{Baltimore, MD, USA} \\
\normalsize{\texttt{\{pkaushi1,schaud35,avaidya7,rchella4,ayuille1\}@jhu.edu} }\\
\normalsize{\url{https://toshi2k2.github.io/share/}}
}

\begin{document}
\maketitle
\begin{abstract}
Adapting large pretrained models to new tasks efficiently and continually is crucial for real-world deployment but remains challenging due to catastrophic forgetting and the high cost of retraining. While parameter-efficient tuning methods like low rank adaptation (LoRA) reduce computational demands, they lack mechanisms for strict continual learning and knowledge integration, without relying on data replay, or multiple adapters. We propose Share, a novel approach to parameter efficient continual finetuning that learns and dynamically updates a single, shared low-rank subspace, enabling seamless adaptation across multiple tasks and modalities. Share constructs a foundational subspace that extracts core knowledge from past tasks and incrementally integrates new information by identifying essential subspace directions. Knowledge from each new task is incorporated into
this evolving subspace, facilitating forward knowledge transfer, while minimizing catastrophic interference. This approach achieves up to 100× parameter reduction and 281× memory savings over traditional LoRA methods, maintaining performance comparable to jointly trained models. A single Share model can replace hundreds of task-specific LoRA adapters, supporting scalable, asynchronous continual learning. Experiments across image classification, natural language understanding, 3D pose estimation, and text-to-image generation validate its effectiveness, making Share a practical and scalable solution for lifelong learning in large-scale AI systems.

\end{abstract}
    
\section{Introduction}
\label{sec:intro}

\begin{figure}[t]
  \centering
  \includegraphics[width=\columnwidth]{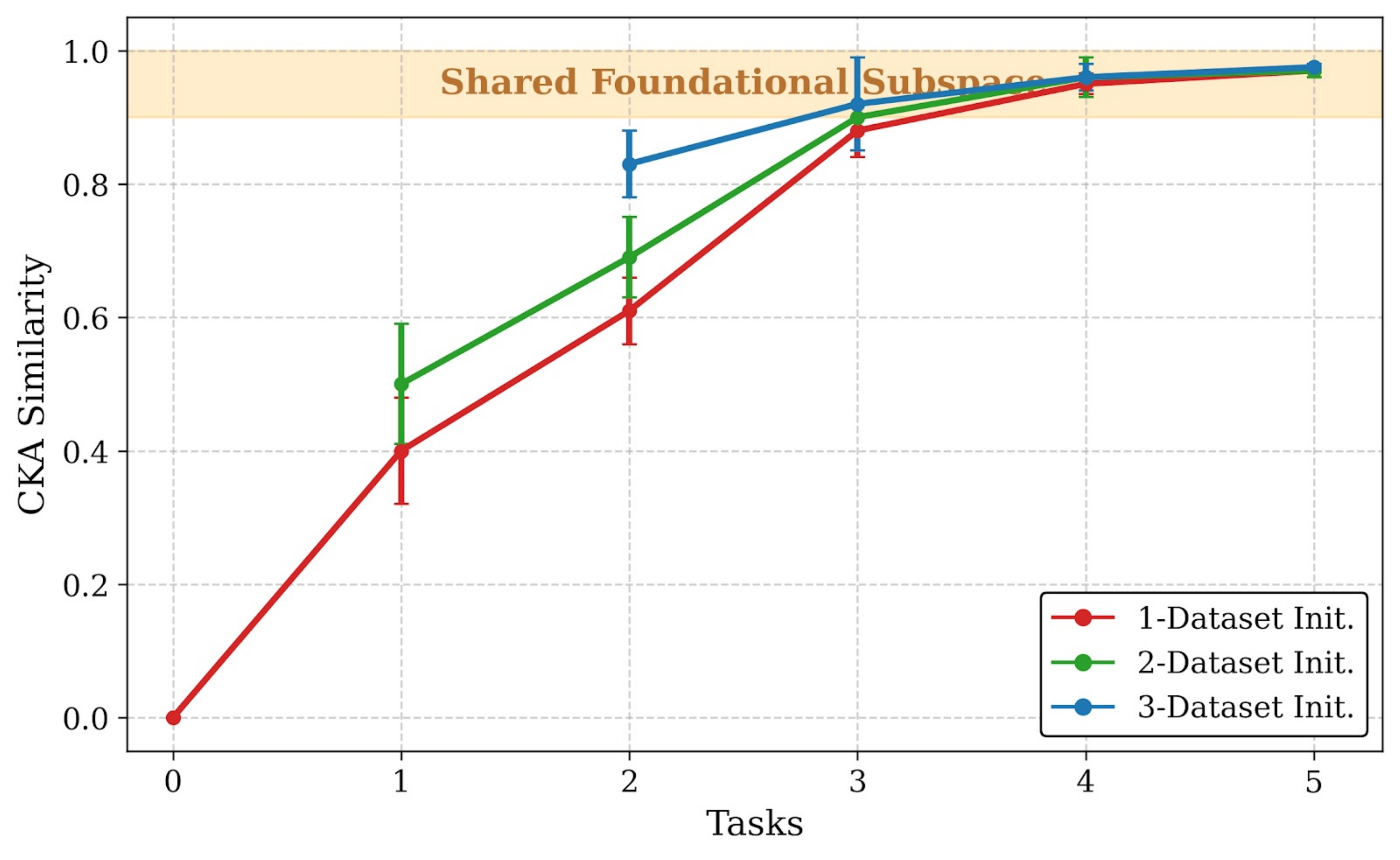}
  \caption{\textbf{Evidence of a Shared Foundational Subspace in Continual Learning.} Linear CKA similarity analysis reveals a universal weight subspace (orange) emerging during sequential learning. Three independent trajectories (red, green, blue), starting from different GLUE task subsets, show monotonic convergence to this shared subspace, reaching near-perfect alignment ($>0.95$) by task $T=5$. Shaded regions show standard deviation across experiments. These results demonstrate: (1) the existence of a common foundational weight subspace that efficiently encodes cross-task knowledge, and (2) our method's ability to discover it through continual adaptation without catastrophic forgetting. This convergence reveals how low-rank adapters naturally bias models toward shared weight structures that generalize across diverse tasks.}
  \label{fig:teaser}
\end{figure}

\begin{figure*}[!htb]
  \centering
  \includegraphics[width=\textwidth]{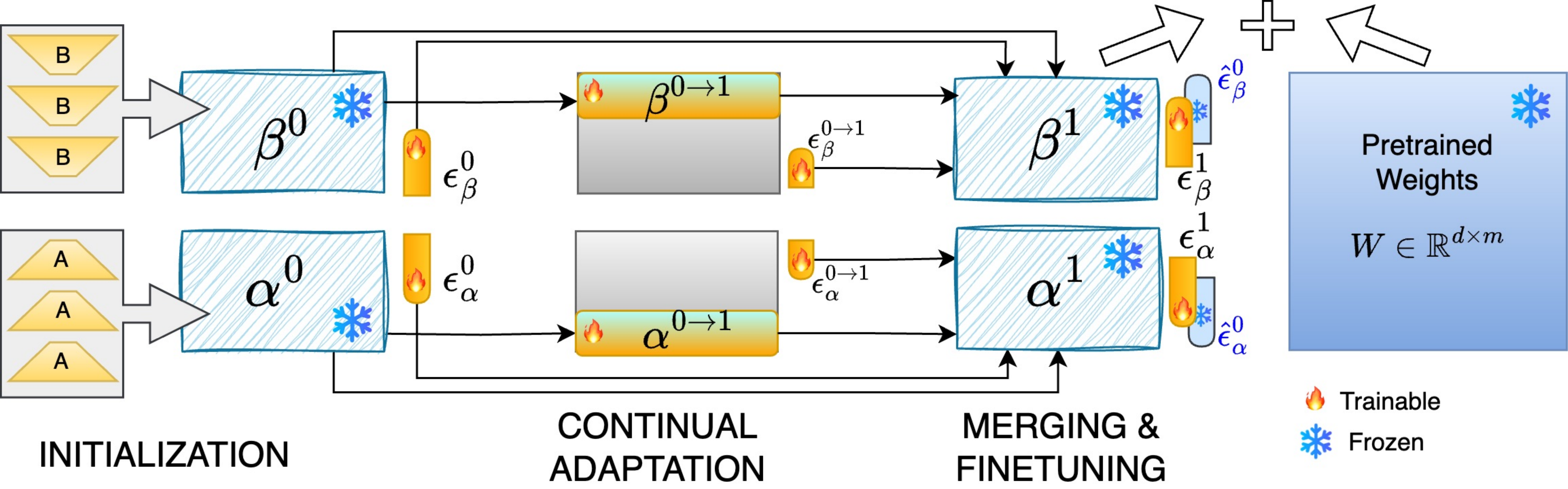}
  \caption{\textbf{Share}. Our continual reparameterization where only principal coefficients $\epsilon^t$ are trained. \textbf{a.} \textbf{Initialization} We initialize the principal factors ($\alpha^{0}, \beta^{0}$) of our Share model using available LoRA~\cite{hu2021lora} adapters ($A, B$). \textbf{b. Continual Adaptation} Few top $\varphi \ll k$ factors, shown as $\alpha^{0\to 1}, \beta^{0 \to 1}$, and temporary coefficients $\epsilon^{0\to 1}$ are fine-tuned when new data is incrementally received. \textbf{Merging \& Fine-tuning} The factors $\alpha^{0}, \beta^{0}$ and temporary factors $\alpha^{0\to 1}, \beta^{0\to 1}$ are merged using the initialization procedure, and $\alpha^1, \beta^1, \epsilon^i_{\alpha,\beta} \quad \forall i \in [0,1]$ are analytically recalculated. $\epsilon^{1}$ can then be further fine-tuned to boost performance.}
  \label{fig:eflux}
\end{figure*}

Adapting large pretrained models, like LLMs, VLMs and Diffusion models, for continual learning presents significant challenges, notably catastrophic forgetting and the substantial resources required for retraining. Traditional fine-tuning methods often require retraining all model parameters, leading to inefficiencies, especially as model sizes increase. As these models scale, they require more resources and memory, making them inaccessible to ordinary researchers and increasing environmental impact~\cite{Ligozat2021UnravelingTH}. Parameter-efficient finetuning techniques, such as LoRA~\cite{hu2021lora}, address some of these issues by introducing trainable low-rank matrices into each layer of the model, effectively reducing the number of parameters that need adjustment during finetuning. However, while LoRA reduces computational demands, it lacks mechanisms for continual learning and knowledge integration, often requiring separate adapters for each task, which can be inefficient and hinders cross-task knowledge sharing which improves robustness and domain generalization~\cite{sanh2021multitask, raffel2020exploring}.

Recent advancements have explored integrating parameter-efficient tuning with continual learning strategies. For instance, methods like O-LoRA~\cite{wang2023orthogonal} propose learning new tasks in orthogonal subspaces to mitigate forgetting. However, these approaches do not fully leverage shared knowledge across tasks, limiting forward (and backward) knowledge transfer, as they require training individual models, or experts, for new tasks. All such methods fall short of \textbf{Strict Continual Learning}~\cite{kaushik2021understandingcatastrophicforgettingremembering}, which requires models to learn continually, without data replay, additional models, or increase in model size, much like humans. Our work tries to remedy this.

Recently, \textit{Universal Weight Subspace Hypothesis}~\cite{kaushik2025universalweightsubspacehypothesis} has proven that neural network weights often converge to layerwise, shared subspace across tasks and datasets, which can be employed for efficient training, inference and model merging. Method like EigenLoRAx~\cite{kaushik2025eigenlorax} have applied this concept for very efficient finetuning achieving equal or better performance to LoRA at fraction of the cost. However, ~\cite{kaushik2025eigenlorax} extract the shared subspace beforehand, and the question of continually improving or learning the shared "universal" subspace is left unanswered. In this work, we show, with theoretical analysis, that our simple method, \textbf{Share}, is capable of approximating the shared subspace in an \textit{almost} strict continual setup.

In this paper, we introduce \textbf{Share}, a novel approach to \textbf{Parameter-Efficient Continual Finetuning (PaCT)} that learns and dynamically updates a shared low-rank subspace, enabling seamless adaptation across multiple tasks and modalities. Share constructs a foundational subspace that captures core knowledge from past tasks and incrementally integrates new information by identifying and expanding essential subspace directions. Each new task is projected into this evolving subspace, facilitating forward knowledge transfer, while older knowledge is analytically reprojected to minimize catastrophic interference. Interestingly, we also observe instances of \textit{backward knowledge transfer} due to presence of this subspace. This approach achieves up to \textbf{100×} parameter reduction and \textbf{281×} memory savings over traditional LoRA methods, maintaining performance comparable to jointly trained models. A single Share model can replace hundreds of task-specific LoRA adapters, supporting scalable, asynchronous continual learning. Experiments across image classification, 3D object pose estimation, natural language understanding, and text-to-image generation validate its effectiveness, making Share a practical and scalable solution for lifelong learning in large-scale AI systems.

To our knowledge, Share is among the earliest works to present a viable solution for Parameter-Efficient Continual Finetuning (for almost strict continual learning) applicable to diverse and complex models.

Our main contributions are as follows:
\begin{itemize}
  \item We introduce a \textbf{replay-free, (almost strict) continual learning} method for large pretrained models that leverages a shared, low-rank foundational subspace of adapters to achieve \textbf{compute and memory-efficient} learning.
  \item Share enables \textbf{continual learning from hybrid streams} of both data and LoRA adapters, seamlessly merging information into a single model.
  \item Share requires orders of magnitude fewer trainable parameters (up to \textbf{100× reduction}) for finetuningand offers up to \textbf{281× memory savings} compared to traditional LoRA methods.
  \item A single set of \textbf{continually learned Share principal factors} can replace hundreds of LoRA adapters, facilitating scalable and efficient model deployment.
  \item We demonstrate Share's applicability across \textbf{various models and modalities}, including image classification, 3D object pose estimation, natural language understanding, commonsense reasoning, math reasoning, and text-to-image generative models.
\end{itemize}

These contributions position Share as a \textbf{scalable and practical} solution for efficient continual learning in large-scale AI systems, addressing critical needs in the deployment of adaptable machine learning models.

\section{Related Work}
\label{sec:related}

\textbf{Efficient Replay Free Continual Learning}\quad While continual learning addresses catastrophic forgetting~\cite{FRENCH1999128}, its application to large models remains challenging, particularly under strict constraints that prohibit data replay and parameter growth~\cite{kaushik2021understandingcatastrophicforgettingremembering}. Recent methods for large models~\cite{liang2024inflora, Yu2024BoostingCL,Smith2023ContinualDW,smith2024continual,wang2023orthogonal} require ever-growing adapter sets and primarily operate as mixture-of-experts systems, limiting their practical utility to specific domains, while also violating the conditions of \textbf{Strict Continual Learning}~\cite{kaushik2021understandingcatastrophicforgettingremembering}, which requires no access to previous data, additional models, or increase in model size. In contrast, \textit{Share} enables true continual learning across diverse architectures and modalities without requiring data replay, and with negligible increase in number of model parameters, almost fulfilling the conditions of the strict setup.\\
\textbf{Model Merging}\quad While recent work has shown promise in merging task-specific models~\cite{marczak2024magmax, kaushik2025eigenlorax, shah2025ziplora,yadav2024ties}, these approaches either focus on narrow domains, lack continual learning capabilities, or require multiple model instances~\cite{zhou2022mixtureofexperts}. \textit{Share} advances this field by enabling efficient, continuous merging of both incoming data and adapters while preserving knowledge across tasks.\\
\textbf{Low-Rank Adaptation}\quad LoRA~\cite{hu2021lora} and its variants~\cite{dora,kopiczko_vera_2023} have made model adaptation more efficient but lack mechanisms for continual knowledge integration. Current scaling solutions either focus on adapter switching~\cite{sheng2023s} or batch optimization~\cite{wen2023batched}, often at the cost of performance~\cite{sharma_laser_2023,gholami2022survey}. \textit{Share} uniquely addresses these limitations through its shared foundational subspace approach, enabling continuous knowledge accumulation while maintaining both efficiency and performance. Notably, \textit{Share} can compress hundreds of adapters into a single set of factors through an elegant data and gradient-free process.

\section{Method}
\paragraph{Problem Setting}
We study \textit{parameter-efficient continual finetuning}, where a pretrained model \( h(W_0,x) \) adapts to a sequence of tasks \( \{\tau_1, \tau_2, \dots, \tau_t\} \) with minimal parameter overhead. At each timestep \( t \), we receive either task-specific data \( \mathcal{S}_t = \{(x_i^t, y_i^t)\}_{i=1}^{S_t} \) or a LoRA adapter \( \Delta W_t = B_t A_t\), with no access to past tasks. The goal is to \textit{continually integrate new knowledge} while minimizing trainable parameters and memory usage, ensuring \textit{knowledge retention}. Here, \( W_0 \) remains \textit{fixed} across tasks, while a newly obtained (or trained) \( \Delta W_t \) for task $\tau_t$ is available only at time step \( t \). This formulation allows \textit{efficient continual adaptation} without catastrophic forgetting or excessive storage costs.

To this end, our method, \textit{Share}, maintains an \textit{evolving low-rank parameter subspace} for each layer of the pretrained model. For a new task, we reuse the basis vectors spanning the subspace and learn task-specific coefficients \( \epsilon^t \) instead of storing separate adapters. We focus on a single layer to further explain the setup and theoretical analysis.

\label{sec:method}
\subsection{Motivation}
\label{ssec:motive}

\paragraph{Motivation}
\textit{Share} is based on the hypothesis that for similar tasks and modalities, \\
``\textit{LoRA adapters of a pretrained model share a common low-rank subspace.}".\\ If we can identify the principal basis vectors of this subspace~\cite{139758}, any new adapter can be expressed as a \textit{linear combination} of these basis vectors, reducing the need for separate adaptation. We validate this in Appendix \cref{fig:pca}, where we initialize all LoRA adapters using \textit{Share}, extract the top \( k \) principal basis vectors via SVD, and reconstruct the original adapters analytically. The reconstruction error and performance evaluation confirm that a single set of \textit{Share} principal basis vectors can approximate all adapters \textit{without significant performance loss}.

Identifying this subspace may initially require a lot of adapters, but in real-world scenarios, we often start with limited or even a single adapter and progressively integrate new ones. \textit{Share} enables \textit{incremental discovery and refinement} of this foundational subspace as more adapters and data become available, making it a scalable and adaptive solution for continual learning.

\subsection{Share: Continual Shared Subspace Adaptation}
\label{ssec:eigen}
\textit{Share} maintains two sets of parameters: \textit{principal basis vectors}, which remain frozen during finetuning, and \textit{task-specific coefficients}, which are learned. The learning process consists of three phases—\textit{initialization, continual adaptation, and merging} \& \textit{fine-tuning}—as illustrated in \cref{fig:eflux}. During \textit{initialization}, an incomplete low-rank subspace of principal basis vectors is formed. In the \textit{continual adaptation} phase, this subspace is refined as new data or LoRA adapters arrive. Finally, in the \textit{merging \& fine-tuning} phase, newly learned basis vectors are integrated with existing ones, updating the principal basis vectors, and optimizing task-specific coefficients.

\subsubsection{Step 1 - Initialization}
\label{method:init}

We initialize the foundational subspace (i.e., principal basis vectors) using \( t \geq 1 \) LoRA adapters, where a larger \( t \) provides a more representative subspace. Given a frozen pretrained weight matrix \( W_0 \in \mathbb{R}^{n\times d} \), LoRA introduces two low-rank trainable matrices, \( B \in \mathbb{R}^{n\times r} \) and \( A \in \mathbb{R}^{r\times d} \), where \( r \) is the LoRA rank. The modified forward pass is $h = W_0 x + \Delta W x = W_0 x + B A x.$
To compute the basis vectors of previously seen $t$ tasks, we reshape the LoRA matrices as stacked rank-\( r \) vectors:
\begin{align*}
  \mathcal{B}^t = [B_1, B_2, \dots, B_t] \in \mathbb{R}^{n \times (t r)}. \label{eq:rankdataB}
\end{align*}
A similar stacked matrix $\mathcal{A}^t$ is calculated for $A$ matrices from $t$ tasks.
We center the matrices $\mathcal{B}^t$ and $\mathcal{A}^t$ and perform SVD on the mean-centered matrices $\overline{\mathcal{B}}_t$ and $\overline{\mathcal{A}}_t$ respectively. This gives us principal basis vectors, $\beta^t$ and $\alpha^t$, respectively. We repeat this for each layer. Then we select the top \( k \) basis vectors based on the highest eigenvalues that span our reusable shared subspace:
\begin{align*}
  \beta^t_{[:k]} \in \mathbb{R}^{n\times k}, \text{\hspace{0.5mm}} \alpha^t_{[:k]} \in \mathbb{R}^{d \times k}
\end{align*}
We keep these \textit{principal basis vectors frozen} during finetuning, while training only the randomly initialized \textit{coefficients} \( \epsilon_{\alpha}, \epsilon_{\beta} \in \mathbb{R}^{k \times p} \), where \( p \) (\textit{pseudo-rank}) can be as small as 1.

This initialization significantly reduces trainable parameters compared to LoRA (e.g., 100× fewer parameters for one task in the GLUE~\cite{glue} experiment, \cref{ssec:glue}), yielding a relative savings of $1 - \frac{k \times p}{(n+d) \times r}$.
The modified forward pass with an initialized \textit{Share} model is:
\begin{equation}
  h_t = W_0 x + (\beta^t \epsilon^t_{\beta})(\alpha^t \epsilon^t_{\alpha})^\top x \quad \forall x\in\mathcal{S}_t
  \label{eq:elora}
\end{equation}

Notably, this initialization is \textit{data- and gradient-free}. If no LoRA adapters are available, \textit{Share} can be initialized by training a LoRA adapter on initial data.

\subsubsection{Step 2 - Continual Adaptation}
After initialization, we receive either new task adapters $\Delta W_{t+1}$ or ar $\mathcal{S}_{t+1}$. If only adapters arrive, we proceed directly to merging. For the latter case, we perform efficient adaptation by adding new basis vectors while preserving the foundational subspace:

\noindent
\textbf{Learning new basis vectors:} To learn new basis vectors at time $t+1$, we initialize $\varphi < k$ temporary basis vectors along with their coefficients as follows:
\begin{align*}
  \beta^{t\rightarrow t+1} &= \beta^{t}_{[:\varphi]} \in \mathbb{R}^{n\times \varphi} \\
  \alpha^{t\rightarrow t+1} &= \alpha^{t}_{[:\varphi]} \in \mathbb{R}^{d\times \varphi} \\
  \epsilon^{t\rightarrow t+1}_\beta, \epsilon^{t\rightarrow t+1}_\alpha &\sim \mathcal{N}(0,\sigma^2) \in \mathbb{R}^{\varphi\times p}
\end{align*}

The modified forward pass becomes:
\begin{align*}
  h = W_0x + (\beta^{t\rightarrow t+1}\epsilon^{t\rightarrow t+1}_\beta)(\alpha^{t\rightarrow t+1}\epsilon^{t\rightarrow t+1}_\alpha
  )^\top x \text{\hspace{1.8mm}} \forall x\in\mathcal{S}_t
\end{align*}

This temporary expansion requires only $\varphi(n + d + 2p)$ parameters, significantly fewer than LoRA's $r(n + d)$ parameters. Both temporary basis vectors and coefficients are optimized before merging.

\subsubsection{Step 3 - Merging and Finetuning}\label{sec:merge_ft}
After continual adaptation, we merge the temporary basis vectors with the foundational subspace while preserving knowledge from all tasks. Below we show the steps to compute $(\beta^{t+1}, \epsilon^{t+1}_\beta)$ which can be identically followed to obtain $(\alpha^{t+1}, \epsilon^{t+1}_\alpha)$.\\
\textbf{Knowledge Integration:} First, reconstruct task-specific adapters using shared basis vectors
$\hat{B_i}=\beta^t \epsilon^i_\beta$ and then stack them along with the new basis vectors learned above:
\begin{align*}
  \hat{\mathcal{B}}^{t+1} = [ \hat{B}_1, \dots, \hat{B}_t, \beta^{t\rightarrow t+1} \epsilon^{t+1}_\beta ] \in \mathbb{R}^{n\times (t+1)p}
\end{align*}
\textbf{Factor Update:} Obtain new shared basis vectors $\beta^{t+1}$ as follows:
\begin{align}
  U_k\Sigma_kV^T_k = \text{SVD}(\hat{\mathcal{B}}^{t+1})_{[:k]} \label{eq:svd_usv} \\
  \beta^{t+1} = U_k,\quad [\epsilon^1_{\beta}, \cdots, \epsilon^{t+1}_{\beta}] = \Sigma_kV^T_k \label{eq:factor_update}
\end{align}
\textbf{Coefficient Update:} Using matrix projection and Moore-Penrose pseudoinverse of $\beta^{t+1}$, we analytically calculate task coefficients $\epsilon^{t+1}_{\beta}$
that minimizes the reconstruction error $\|\hat{B}_i - \beta^{t+1} \epsilon^{t+1}_{\beta}\|_F^2$ as follows:
\begin{align}
  \epsilon^i_{\beta} &= ((\beta^{t+1})^T\beta^{t+1})^{-1} (\beta^{t+1})^T \hat{B}_i
\end{align}
When $\beta^{t+1}$ has orthonormal columns, this simplifies to $\epsilon^i_{\beta} = (\beta^{t+1})^T \hat{B}_i$ which can optionally be further finetuned for enhanced performance if the continual assumptions are relaxed to access some data from previous tasks (referred as \textit{Share}-full).
This gradient-free process yields a single set of \textit{Share} basis vectors and $t$ sets of lightweight coefficients.
The storage advantage becomes pronounced as $t$ grows, since $k,p \ll r$ in practice.
A detailed description is provided in Appendix ~\cref{algo:share}.

\subsubsection{Ablation: How to decide $k$, $p$ and $\varphi$?}\label{sssec:hyper}
The number of \textit{Share} basis vectors $k$ is determined by a threshold based on the explained variance of the factor data matrix $D$. We find that $k$ with as low as $60\%$ explained variance is effective. Other methods~\cite{pca_bayes, gavishsl} using eigenvalues of $D$ can also determine $k$. Our analysis shows that $\varphi=[1,k/4]$ is effective for identifying new basis vectors in the foundational subspace. A pseudo-rank of $p=1$ is effective, with higher values yielding minimal additional benefits. Starting with $p=r/3$ is advisable. A detailed ablation of these selections is provided in the appendix \cref{sec:ablation_app}.

\begin{table*}[t]
  \centering
  \small
  \renewcommand{\arraystretch}{1.1}
  \setlength{\tabcolsep}{4.2pt}

  \begin{tabular}{@{} l cc c llllll c @{}}
    \toprule
    \textbf{Method} & \textbf{Params} $\downarrow$ & \textbf{Size (MB)} $\downarrow$ & \textbf{Task} & \textbf{COLA} & \textbf{MRPC} & \textbf{RTE} & \textbf{STSB} & \textbf{QNLI} & \textbf{SST-2} & \textbf{Avg.} \\
    \midrule
    Upper Bound & 125M & 500 & -- & 59.91 & 89.01 & 79.70 & 90.90 & 92.31 & 91.28 & 83.90 \\
    LoRA (non-CL) & 1.2M$\times$6 & 81.6 & -- & 59.56 & 86.76 & 77.61 & 90.81 & 92.53 & 93.35 & 83.43 \\
    \midrule
    \multirow{6}{*}{\textit{Share} (CL)} & \multirow{6}{*}{\textbf{0.012M}} & \multirow{6}{*}{\textbf{0.29}}
    & T-0 & 56.00 & -- & -- & -- & -- & -- & -- \\
    & & & T-1 & 55.54 \scalebox{0.7}{\color{red}{$\downarrow$0.46}} & 68.38 & -- & -- & -- & -- & -- \\
    & & & T-2 & 56.50 \scalebox{0.7}{\color{ForestGreen}{$\uparrow$0.96}} & 68.38 & 73.29 & -- & -- & -- & -- \\
    & & & T-3 & 56.50 & 68.38 & 73.29 & 88.91 & -- & -- & -- \\
    & & & T-4 & 56.24 \scalebox{0.7}{\color{red}{$\downarrow$0.26}} & 68.38 & 73.65 \scalebox{0.7}{\color{ForestGreen}{$\uparrow$0.36}} & 88.91 & 91.84 & -- & -- \\
    & & & T-5 & \bfseries 55.99 \scalebox{0.7}{\color{red}{$\downarrow$0.51}} & \bfseries 68.38 & \bfseries 73.29 \scalebox{0.7}{\color{red}{$\downarrow$0.36}} & \bfseries 88.91 & \bfseries 91.95 \scalebox{0.7}{\color{ForestGreen}{$\uparrow$0.11}} & \bfseries 93.58 & \bfseries 78.69 \\
    \midrule
    \textit{Share}-full & \textbf{0.012M} & \textbf{0.29} & -- & \textbf{59.81} & \textbf{86.99} & \textbf{77.62} & \textbf{90.80} & \textbf{92.66} & \textbf{93.39} & \textbf{83.44} \\
    \bottomrule
  \end{tabular}
  \caption{Continual GLUE benchmark results. We report Matthews correlation for CoLA, Pearson correlation for STS-B, and accuracy for remaining tasks. Results shown as mean ± std across 3 random seeds. Higher values indicate better performance. Experiments use a data-only setup (no initial LoRA).\tablefootnote{All Share results significantly outperform catastrophic forgetting baselines ($p < 0.01$, paired t-test).} $\downarrow$ denotes forgetting and $\uparrow$ denotes backward transfer relative to the task's historical peak. \textit{Share} achieves significant parameter efficiency while maintaining high average performance.}
  \label{tab:glue_benchmark_results}
\end{table*}

\subsection{Theoretical Analysis}
\label{ssec:theory}

For brevity, we use a general $D_t \in \mathbb{R}^{n_t \times d}$ represents the LoRA parameters ($B_t$ or $A_t$) for task $\tau_t$. Let $h_{t} = h(D_t,x) = W_0x + D_tx$ be an independent empirical risk minimizer (ERM) of task $\tau_t$ with weights $D_t$ of rank-$n_t$.

\begin{proposition}
  \label{thm:inc}
  \textbf{(Incremental Subspace Error Bound)}: Let $\mathcal{D}^{t} = [D_1, D_2, \ldots, D_t] \in \mathbb{R}^{N_t \times d}$ be cumulatively stacked weight matrix up to task $\tau_t$, where $N_t = \sum_{i=1}^t n_i$.
  Using the Share approximation at task $t$ with $k$ principal basis vectors we get $\hat{\mathcal{D}}^{t}$ be the rank-$k$ SVD of $\mathcal{D}^{t}$ such that $\hat{\mathcal{D}}^{t}=\text{SVD}(\mathcal{D}^{t})_{[:k]}$.
  Then the reconstruction error satisfies:
  $$\|\mathcal{D}^{t} - \hat{\mathcal{D}}^{t}\|_F^2 = \sum_{i=k+1}^{\min(N_t, d)} (\sigma_i^{(t)})^2,$$ where $\sigma_i$'s are the singular values of $\mathcal{D}^{t}$ for the non-principal basis vectors. See proof in appendix \cref{sec:theory_supp}.

\end{proposition}

\begin{theorem}\label{theorem:upper_bound}
  Let $\hegn\in\mathcal{H}_k : \hegn=h(\Wegn,x)$ denote rank-k ERM while finetuning for a new task $\tau_{t+1}$ with true solution $\htrue\in \mathcal{H}:\htrue=h(\Wtrue,x)$ of rank-$n_{t+1}$. Here $\epsilon$ with rank-$k$ is trainable and $\Vk$ are fixed right singular vectors from $\text{SVD}(\mathcal{D}^t)_{[:k]}$. Given a Lipschitz continuous loss ($\loss(\cdot)$) that is strong convex over the weight spaces spanned by $D$ and $\Vk$ with some Lipschitz constant ($L$), we say with probability at least $1-4\delta$ for some $\delta>0$,
  \begin{equation}
    \label{equation:bound_baseline}
    \Vert \Wtrue - \Wnew\lVert^2_F \leq \sqrt{\frac{C_1n_{t+1}}{\size}}+\sqrt{\frac{2\ln(1/\delta)}{\size}} + C_2
  \end{equation}
  \begin{equation}
    \label{equation:bound_ours}
    \Vert \Wtrue - \Wegn \lVert^2_F \leq \sqrt{\frac{C_1k}{\size}}+\sqrt{\frac{2\ln(1/\delta)}{\size}} + \singularsum + C_2
  \end{equation}
  $\sigma_{k+1}, \dots, \sigma_{n_{t+1}}$ denote singular values of $\Wtrue V_{\perp}$, where $V_{\perp}$ contains right singular vectors of $\mathcal{D}^t$ orthogonal to $V_k$. Here, we assuming $\lVert D \rVert_F \leq B$ and normalized input data. Here $S_{t+1} = |\mathcal{S}_{t+1}|$ refers to dataset size for task $\tau_{t+1}.$
\end{theorem}
Theorem \ref{theorem:upper_bound} gives an upper bound on the Frobenius-norm error of $\Wegn$ and $\Wnew$ with respect tosh $\Wtrue$. The refined bound in \cref{equation:bound_ours} is tighter when task $\tau_{\tnew}$ lies largely in the shared principal subspace, as reflected in the second term involving the truncated singular values of $\Wnew$. In this regime, the first term $\sqrt{k/\size}$ dominates. When $n_{t+1} \ge k$, the bound in \ref{equation:bound_ours} is therefore tighter than that in \ref{equation:bound_baseline}. Likewise, when $n_{t+1} \le k$, the bound for $\Wnew$ is tighter than that for $\Wegn$.

If $\Wtrue$ has substantial mass outside the estimated top-$k$ singular directions, the second term in \ref{equation:bound_baseline} becomes significant, making direct comparison of the two bounds harder. Yet in such cases, $\Wegn$ cannot converge well because it is confined to the top-$k$ components of $\hat{\mathcal{D}}^{(t)}$, whereas $\Wnew$ can still perform substantially better as it can express beyond this subspace. Proof in appendix \cref{sec:theory_supp}.

\section{Experimental Evaluation}
\label{sec:exp}

We evaluate \textit{Share} across diverse vision and language tasks using two rigorous baselines: (1) \textbf{non-CL} LoRA, where separate adapters are independently trained per task to avoid catastrophic forgetting~\cite{kaushik2021understandingcatastrophicforgettingremembering}, and (2) \textbf{Joint} LoRA, trained simultaneously on all tasks, representing theoretical performance upper bounds. Experiments span multiple architectures (ViT~\cite{vision_transformer}, RoBERTa~\cite{roberta}, LLaMA~\cite{llama}, Mistral~\cite{jiang2023mistral7b}, Flux~\cite{flux}), demonstrating \textit{Share}'s diverse applicability.

For direct comparisons with existing parameter-efficient continual learning methods, we focus on image classification and 3D pose estimation tasks, where established benchmarks exist. We include comparisons with prompt-based methods that, while avoiding data storage, still require maintaining task-specific parameters. Broader comparisons with other PeFT methods are deferred to future work due to their limited continual learning capabilities.

We implement two realistic continual learning scenarios: (1) \textit{data-only}, where a single adapter initializes the system followed by streaming data, and (2) \textit{hybrid}, which integrates both adapters and data incrementally. These configurations reflect real-world deployment constraints where historical data access is limited. For each experiment, we analyze parameter count, memory requirements, and temporary expansion during adaptation. We also evaluate \textbf{\textit{Share}-full}, which finetunes the coefficients after Step 3 in \cref{sec:merge_ft} to maximize performance.

Our results demonstrate that \textit{Share} achieves up to $100\times$ parameter reduction and $281\times$ memory savings while maintaining performance comparable to upper-bound baselines, despite learning in a replay-free continual manner. \cref{ssec:scale} shows \textit{Share}'s capability for efficient continual model merging and serving multiple LoRAs at scale, with detailed ablations provided in the supplementary material.
\begin{table*}[!htb]
  \centering

  \renewcommand{\arraystretch}{1}
  \begin{adjustbox}{width=\textwidth,center}
    \begin{tabular}{lrcccccccc}
      \toprule
      \multirow{2}{*}{\textbf{Method}} & \multirow{2}{*}{\textbf{\# Params}} & \multicolumn{2}{c}{\textbf{CIFAR-100}} & \multicolumn{2}{c}{\textbf{Food-100}} & \multicolumn{2}{c}{\textbf{Caltech-100}} & \multicolumn{2}{c}{\textbf{Flowers-100}} \\
      \cmidrule{3-10}
      & & \textbf{Acc.} & \textbf{Forg.↓} & \textbf{Acc.} & \textbf{Forg.↓} & \textbf{Acc.} & \textbf{Forg.↓} & \textbf{Acc.} & \textbf{Forg.↓} \\
      \midrule

      Upper Bound & 86M & 94.20 {\scriptsize$\pm$ 0.8} & -- & 90.40 {\scriptsize$\pm$ 1.2} & -- & 98.32 {\scriptsize$\pm$ 0.2} & -- & 98.83 {\scriptsize$\pm$ 0.8} & -- \\
      \midrule
      Full-seq & 86M & 30.39 {\scriptsize$\pm$ 1.9} & 67.53 {\scriptsize$\pm$ 0.5} & 26.90 {\scriptsize$\pm$ 0.5} & 68.97 {\scriptsize$\pm$ 1.1} & 27.04 {\scriptsize$\pm$ 1.2} & 73.12 {\scriptsize$\pm$ 1.3} & 35.64 {\scriptsize$\pm$ 1.9} & 69.21 {\scriptsize$\pm$ 0.2} \\

      Linear-seq & 0.08M & 68.43 {\scriptsize$\pm$ 0.1} & 21.93 {\scriptsize$\pm$ 0.1} & 60.58 {\scriptsize$\pm$ 0.3} & 23.66 {\scriptsize$\pm$ 0.3} & 76.17 {\scriptsize$\pm$ 0.1} & 25.78 {\scriptsize$\pm$ 0.4} & 78.95 {\scriptsize$\pm$ 0.6} & 17.34 {\scriptsize$\pm$ 0.6} \\
      \midrule
      EWC & 86M & 59.60 {\scriptsize$\pm$ 1.3} & 23.73 {\scriptsize$\pm$ 3.0} & 55.27 {\scriptsize$\pm$ 1.1} & 26.34 {\scriptsize$\pm$ 3.4} & 57.96 {\scriptsize$\pm$ 1.8} & 29.26 {\scriptsize$\pm$ 1.6} & 69.79 {\scriptsize$\pm$ 1.8} & 24.65 {\scriptsize$\pm$ 0.1} \\

      LwF & 86M & 68.22 {\scriptsize$\pm$ 1.6} & 25.14 {\scriptsize$\pm$ 2.4} & 60.15 {\scriptsize$\pm$ 0.7} & 24.63 {\scriptsize$\pm$ 0.8} & 63.26 {\scriptsize$\pm$ 1.4} & 33.42 {\scriptsize$\pm$ 1.9} & 71.78 {\scriptsize$\pm$ 2.0} & 15.44 {\scriptsize$\pm$ 1.5} \\

      L2P & 86M & 83.05 {\scriptsize$\pm$ 1.0} & 2.12 {\scriptsize$\pm$ 0.1} & 70.48 {\scriptsize$\pm$ 1.4} & 4.86 {\scriptsize$\pm$ 0.5} & 89.34 {\scriptsize$\pm$ 1.8} & 6.93 {\scriptsize$\pm$ 0.4} & 94.53 {\scriptsize$\pm$ 1.2} & 7.21 {\scriptsize$\pm$ 0.2} \\

      DualPrompt & 0.35M & 84.77 {\scriptsize$\pm$ 0.7} & 1.40 {\scriptsize$\pm$ 0.3} & 75.31 {\scriptsize$\pm$ 0.6} & 2.85 {\scriptsize$\pm$ 0.3} & 91.52 {\scriptsize$\pm$ 0.9} & 4.84 {\scriptsize$\pm$ 0.1} & 95.25 {\scriptsize$\pm$ 0.8} & 5.60 {\scriptsize$\pm$ 0.4} \\

      CODA-P & 0.42M & 86.25 {\scriptsize$\pm$ 0.7} & 1.16 {\scriptsize$\pm$ 0.2} & 77.58 {\scriptsize$\pm$ 1.2} & 2.73 {\scriptsize$\pm$ 0.5} & 91.44 {\scriptsize$\pm$ 0.3} & 4.03 {\scriptsize$\pm$ 2.1} & 97.02 {\scriptsize$\pm$ 0.4} & 4.67 {\scriptsize$\pm$ 0.3} \\

      DAP & 0.19M & 94.05 {\scriptsize$\pm$ 1.2} & 0.41 {\scriptsize$\pm$ 0.1} & 88.37 {\scriptsize$\pm$ 0.6} & 0.92 {\scriptsize$\pm$ 0.1} & 97.23 {\scriptsize$\pm$ 0.3} & 2.52 {\scriptsize$\pm$ 0.8} & 96.49 {\scriptsize$\pm$ 0.1} & 2.28 {\scriptsize$\pm$ 1.0} \\

      EASE & 0.28M & 87.81 {\scriptsize$\pm$ 0.8} & 0.52 {\scriptsize$\pm$ 0.0} & 88.82 {\scriptsize$\pm$ 1.5} & 1.41 {\scriptsize$\pm$ 0.6} & 96.54 {\scriptsize$\pm$ 0.5} & 3.48 {\scriptsize$\pm$ 1.5} & 97.53 {\scriptsize$\pm$ 0.2} & 5.40 {\scriptsize$\pm$ 1.0} \\

      \textbf{Share} & \textbf{0.10M} & \textbf{94.20 }{\scriptsize$\pm$ 0.9} & \textbf{0.40} {\scriptsize$\pm$ 0.4} & \textbf{90.10} {\scriptsize$\pm$ 1.2} & \textbf{0.7} {\scriptsize$\pm$ 0.9} & \textbf{97.70 }{\scriptsize$\pm$ 1.1} & \textbf{2.18} {\scriptsize$\pm$ 2.4} & \textbf{97.90 }{\scriptsize$\pm$ 0.7} & \textbf{2.33} {\scriptsize$\pm$ 1.2} \\

      \bottomrule
    \end{tabular}
  \end{adjustbox}
  \caption{\textbf{Comprehensive Evaluation of Continual Learning Methods on Image Classification Tasks}}
  \caption*{\footnotesize\textit{Params refer to trainable parameters per task stream in continual setting with ViT-B/16 backbone (86M parameters). Full-seq, EWC, and LwF fine-tune all parameters, while other methods use parameter-efficient strategies. Acc.: final accuracy after learning all tasks; Forg.↓: forgetting rate (lower is better).}}
  \label{tab:vision_models}
\end{table*}
\subsection{Continual Natural Language Understanding}
\label{ssec:glue}

\paragraph{Setup}
We evaluate \textbf{Share} on the Continual GLUE benchmark~\cite{glue} using RoBERTa$_{\text{base}}$~\cite{roberta} across six tasks (MRPC, SST-2, CoLA, QNLI, RTE, STS-B) in a challenging replay-free continual learning setting. Following prior work~\cite{kopiczko_vera_2023}, we exclude time-intensive MNLI and QQP tasks. This data-only configuration initializes Share with a single LoRA adapter, testing its ability to build a foundational subspace from minimal initialization. During adaptation, Share temporarily expands to 450K parameters~(Sec.~\ref{sec:method}), still significantly fewer than LoRA. Hyperparameters include LoRA rank $r=32$, Share factors $k=32$, pseudo-rank $p=8$, and temporary factors $\varphi=4$, following Sec.~\ref{sssec:hyper}.\\
\textbf{Analysis}
Table~\ref{tab:glue_benchmark_results} demonstrates Share's exceptional efficiency-performance tradeoff. The continual learning progression shows consistent knowledge retention as Share's foundational subspace evolves. Notably, we observe evidence of \textbf{backward knowledge transfer}, where performance on early tasks (e.g., CoLA improving from 56.00 to 59.81) benefits from learning subsequent tasks---a phenomenon rarely achieved in parameter-efficient continual learning. This bidirectional knowledge transfer emerges from Share's dynamic subspace refinement, where the analytical recalculation of coefficients allows earlier tasks to leverage representations discovered through later learning.

Remarkably, Share-full achieves 83.44\% average performance with only 0.012M parameters (0.29MB), marginally surpassing non-continual LoRA's 83.43\% which requires 7.2M parameters (81.6MB) across six separate adapters. This represents a \textbf{100$\times$ reduction in trainable parameters} and an unprecedented \textbf{281$\times$ memory savings} while maintaining competitive performance.

\cref{fig:icremental-comps} in appendix illustrates how Share rapidly converges to an effective foundational subspace. The analytical knowledge integration enables more stable learning compared to LoRA's potentially noisy gradient-based optimization, explaining Share-full's slight performance advantage despite using orders of magnitude fewer parameters.

\subsection{Continual Image Classification}
\label{sssec:class}

\paragraph{Setup}
Following established continual learning benchmarks~\cite{dap, wang2022dualprompt, zhang2024hyperadapter}, we evaluate Share on four standard datasets: CIFAR-100~\cite{cifar100}, Food-101~\cite{food101}, Caltech-101, and Flowers-102~\cite{flowers102}. Each dataset is restricted to 100 classes and divided into 10 tasks of 10 classes each. We implement a data-only setup using ViT-B/16~\cite{vision_transformer} pretrained on ImageNet-21K as the backbone. For \textit{Share}, we set $k=10$, $p=1$, $\varphi=2$ and finetune for 10 epochs per task.

We compare against both naive baselines (Full-seq, Linear-seq) and state-of-the-art rehearsal-free methods, including regularization-based approaches (EWC~\cite{Kirkpatrick_2017}, LwF~\cite{li2017learningforgetting}) and architecture-based methods spanning prompts (L2P~\cite{l2p}, DAP~\cite{dap}, CODA-Prompt~\cite{Smith_2023_CVPR}, DualPrompt~\cite{wang2022dualprompt}) and adapters (EASE~\cite{zhou2024expandablesubspaceensemblepretrained}). Following~\cite{dap}, we report final accuracy (Acc.) and forgetting rate.\\
\textbf{Analysis}
Table~\ref{tab:vision_models} reveals Share's superior performance across all datasets. On CIFAR-100, Share achieves 94.20\% accuracy with only 0.10M parameters—matching the upper bound while outperforming the next best method (DAP) which requires nearly twice the parameters (0.19M). Share also consistently demonstrates the lowest forgetting rates, despite using fewer parameters than all competing architecture-based methods.

Notably, Share surpasses prompt-based methods that require storing task-specific prompts and adapter approaches that accumulate separate adapters per task. While prompt-based methods like DualPrompt and CODA-P show competitive performance on some datasets, they require more parameters than Share. This efficiency advantage becomes increasingly significant as the number of tasks grows, highlighting Share's scalability for long task sequences.

Unlike specialized methods optimized specifically for image classification, Share maintains this performance while demonstrating versatility across diverse domains—a capability not established for existing approaches.

\subsection{Continual 3D Object Pose Estimation}
\label{ssec:pose}

\begin{table}[htb]
  \centering
  \small %
  \renewcommand{\arraystretch}{1.1}
  \setlength{\tabcolsep}{8pt} %

  \begin{tabular}{@{} l r *{4}{S[table-format=2.2]} @{}}
    \toprule
    \textbf{Method} & \textbf{Params} $\downarrow$ & \textbf{P3D} & \textbf{L1} & \textbf{L2} & \textbf{L3} \\
    \midrule
    Upper Bound & 30M & 88.10 & 73.20 & 58.40 & 37.80 \\
    \midrule
    \rowcolor{gray!5} \multicolumn{6}{@{}l}{\textit{Data-Replay Methods}} \\
    LwF$^\dagger$   & 25M & 53.47 & 44.58 & 39.77 & 36.61 \\
    ICaRL$^\dagger$ & 25M & 57.74 & 44.03 & 38.15 & 33.52 \\
    iNeMO$^\dagger$ & 25M & 79.28 & 64.71 & 52.26 & 34.01 \\
    \midrule
    \textbf{Share}  & \bfseries 1M & \bfseries 81.80 & \bfseries 69.11 & \bfseries 55.60 & \bfseries 35.50 \\
    \bottomrule
  \end{tabular}
  \caption{Continual 3D Pose Estimation on Pascal3D+ and Occluded Pascal3D+ using $\pi/6$ accuracy. Methods marked with $^\dagger$ utilize data-replay. \textit{Share} achieves competitive performance with significantly fewer parameters.}
  \label{tab:pose_estimation}
\end{table}
\begin{figure*}[h]
  \centering
  \includegraphics[width=.8\linewidth]{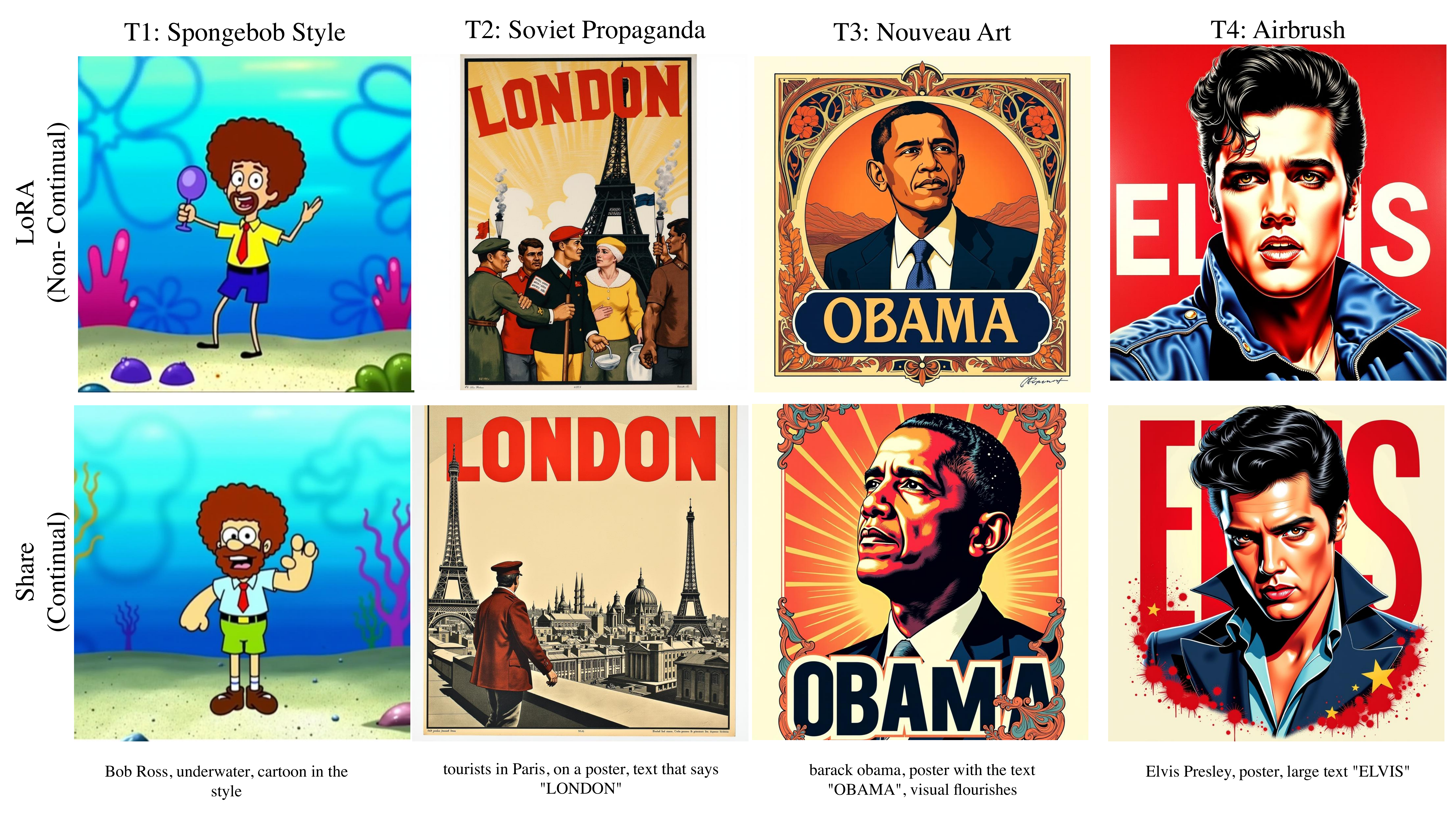} %
  \caption{Comparing continually finetuned \textit{Share} results with individual LoRAs on different tasks for text-to-image generation}
  \label{fig:Diffusion}
\end{figure*}

\paragraph{Setup}
We evaluate Share on continual 3D pose estimation~\cite{fischer2024inemoincrementalneuralmesh} using Pascal3D+~\cite{xiang_wacv14} (natural images of 12 rigid object categories) and Occluded Pascal3D+~\cite{kaushik2024sourcefree,Yuan_2025_ICCV} (introducing progressive occlusion levels L1-L3). This task requires estimating 3D rotation parameters (azimuth, elevation, and in-plane rotation) without explicit 3D information. Following~\cite{fischer2024inemoincrementalneuralmesh}, we measure accuracy using the $\pi/6$ metric—the percentage of predictions where the geodesic distance between predicted and ground truth rotation matrices is below $\pi/6$ radians. Performance is assessed across four occlusion levels (L0-L3) using a pretrained ResNet50 backbone. Share operates with only 1M parameters ($k=16$, $\varphi=4$) under strictly replay-free conditions, compared to baseline methods requiring 25M parameters and exemplar storage.\\

\noindent
\textbf{Analysis}
\cref{tab:pose_estimation} reveals Share's superior performance across all occlusion levels, consistently outperforming replay-based methods including iNeMO despite operating without exemplar storage. This demonstrates Share's ability to maintain effective geometric representations under visual degradation through its foundational subspace approach. The 96\% parameter reduction while achieving enhanced occlusion robustness establishes Share as an efficient solution for real-world deployment scenarios where objects are frequently partially occluded. These results extend parameter-efficient continual learning beyond classification to complex spatial reasoning tasks, preserving geometric relationships across task boundaries without compromising performance.
\begin{table}[htb]
  \centering
  \small
  \renewcommand{\arraystretch}{1.1}
  \setlength{\tabcolsep}{7pt}

  \begin{tabular}{@{} l *{4}{S[table-format=2.3]} @{}}
    \toprule
    \textbf{Method} & \textbf{T1} & \textbf{T2} & \textbf{T3} & \textbf{T4} \\
    \midrule
    \textit{Share} (CL) & 23.227 & {--} & {--} & {--} \\
    & 23.490 \scalebox{0.7}{\color{ForestGreen}{$\uparrow$.26}} & 22.374 & {--} & {--} \\
    & 23.281 \scalebox{0.7}{\color{red}{$\downarrow$.21}} & 22.699 \scalebox{0.7}{\color{ForestGreen}{$\uparrow$.32}} & 24.121 & {--} \\
    & \bfseries 23.382 \scalebox{0.7}{\color{ForestGreen}{$\uparrow$.10}} & 22.400 \scalebox{0.7}{\color{red}{$\downarrow$.30}} & \bfseries 24.436 \scalebox{0.7}{\color{ForestGreen}{$\uparrow$.31}} & \bfseries 21.432 \\
    \midrule
    LoRA (non-CL) & 22.776 & 22.985 & 24.079 & 21.089 \\
    \bottomrule
  \end{tabular}
  \caption{CLIP scores across sequential tasks (T1–T4) for continual text-to-image generation. Results are the mean of 3 random seeds. $\uparrow$ and $\downarrow$ indicate backward transfer and forgetting relative to the previous step, respectively. \textit{Share} demonstrates significant backward knowledge transfer ($p < 0.01$).}
  \label{tab:clip_scores}
\end{table}
\begin{table*}[!htb]
  \centering
  \renewcommand{\arraystretch}{1.2} %
  \setlength{\tabcolsep}{4pt} %
  \arrayrulecolor{black} %
  \resizebox{\textwidth}{!}{%
    \begin{tabular}{lccccccccccc}
      \toprule
      & \textbf{T1} & \textbf{T2} & \textbf{T3} & \textbf{T4} & \textbf{T5} & \textbf{T6} & \textbf{T7} & \textbf{T8} & \textbf{T9} & \textbf{T10} & \textbf{T11} \\
      \midrule
      \textbf{T-0}  & \cellcolor{highlight} 83.70 (\textbf{1.00}) &  &  &  &  &  &  &  &  &  &  \\
      \textbf{T-1}  & 82.42 (0.98) & \cellcolor{highlight} 63.34 (\textbf{1.00}) &  &  &  &  &  &  &  &  &  \\
      \textbf{T-2}  & 79.46 (0.94) & 62.29 (0.98) & \cellcolor{highlight} 86.96 (\textbf{1.00}) &  &  &  &  &  &  &  &  \\
      \textbf{T-3}  & 77.03 (0.92) & 62.25 (0.98) & 86.12 (0.99) & \cellcolor{highlight} 93.69 (\textbf{1.00}) &  &  &  &  &  &  &  \\
      \textbf{T-4}  & 76.58 (0.91) & 62.09 (0.98) & 85.87 (0.98) & 93.25 (0.99) & \cellcolor{highlight} 75.89 (\textbf{1.00}) &  &  &  &  &  &  \\
      \textbf{T-5}  & 76.16 (0.90) & 61.88 (0.97) & 85.37 (0.98) & 93.19 (0.99) & 75.92 (1.00) & \cellcolor{highlight} 45.64 (\textbf{1.00}) &  &  &  &  &  \\
      \textbf{T-6}  & 76.57 (0.91) & 61.66 (0.97) & 84.98 (0.98) & 92.98 (0.99) & 76.00 (1.00) & 45.66 (1.00) & \cellcolor{highlight} 58.27 (\textbf{1.00}) &  &  &  &  \\
      \textbf{T-7}  & 75.85 (0.90) & 61.73 (0.97) & 84.73 (0.97) & 92.64 (0.99) & 73.61 (0.97) & 45.65 (1.00) & 57.78 (0.99) & \cellcolor{highlight} 80.68 (\textbf{1.00}) &  &  &  \\
      \textbf{T-8}  & 75.87 (0.91) & 61.68 (0.97) & 84.75 (0.97) & 92.12 (0.98) & 73.61 (0.97) & 45.67 (1.00) & 54.43 (0.93) & 80.20 (0.99) & \cellcolor{highlight} 52.84 (\textbf{1.00}) &  &  \\
      \textbf{T-9}  & 75.34 (0.90) & 61.65 (0.97) & 84.67 (0.97) & 92.10 (0.98) & 73.68 (0.97) & 45.62 (0.99) & 54.89 (0.94) & 79.99 (0.99) & 51.86 (0.98) & \cellcolor{highlight} 80.10 (\textbf{1.00}) \\
      \textbf{T-10} & \cellcolor{highlight}76.23 \textbf{(0.91)} & \cellcolor{highlight}61.52 \textbf{(0.97)} &\cellcolor{highlight} 84.67 \textbf{(0.97)} &\cellcolor{highlight} 91.52 \textbf{(0.97)} &\cellcolor{highlight} 72.07 \textbf{(0.95)} & \cellcolor{highlight} 45.62 \textbf{(0.99)} & \cellcolor{highlight} 56.14 \textbf{(0.96)} & \cellcolor{highlight}80.14 \textbf{(0.99)} &\cellcolor{highlight} 51.45 \textbf{(0.97)} & \cellcolor{highlight}80.18 \textbf{(1.00)} & \cellcolor{highlight} 42.71 (\textbf{1.00}) \\
      \bottomrule
    \end{tabular}%
  }
  \caption{\textbf{Continual Learning with Lots of LoRAs.} We report the absolute and relative Rouge-L scores at each time-step. Diagonal elements (1.00) are highlighted for clarity.}
  \label{tab:LoLA}
\end{table*}

\subsection{Text to Image Generation}
\label{ssec:flux}
In this experiment, we demonstrate the \textit{Share}'s ability to effectively manage complex tasks and multimodal architectures, such as Text-to-Image Models. As far as we are aware, ours is the one of the earliest work for doing Parameter Efficient Continual Learning for such models without requiring data replay or a mixture of experts like setup.\\
\textbf{Setup} We design this experiment as a hybrid task. We choose Flux~\cite{flux} as our pretrained text-to-image generation model. We choose $4$ continual tasks represented by LoRA adapters or associated prompt-image data. For comparison, we compare with a non-CL LoRA. For initializing \textit{Share}, we utilize the publicly available LoRA adapter from the HuggingFace community library~\cite{von-platen-etal-2022-diffusers}. Hyperparameter: LoRA $r=32$, \textit{Share} $k=32, p=8$ and the pretrained model is Flux~\cite{flux}. \\
\textbf{Analysis}
In the absence of quantitative measures, we present qualitative results for individual LoRAs and continual \textit{Share} in \cref{fig:Diffusion}. The CL tasks span various art styles, with additional results and details in the Appendix. Our findings demonstrate that continual training with \textit{Share} achieves performance comparable to individual LoRAs. Notably, we observe up to a 20× reduction in model size compared to 20 LoRAs, with further reductions as the number of tasks increases. Crucially, a single set of \textit{Share} factors, equal to one adapter, is used for all tasks in \cref{fig:Diffusion}.

\subsection{Continual Asynchronous Learning and Serving of LoRAs at Scale}
\label{ssec:scale}
\paragraph{Setup} We introduce a novel setup for Continual Model Merging and Learning at Scale, inspired by the continual model merging~\cite{marczak2024magmax} and LoRA scaling~\cite{sheng2023s} literature. This setup enables efficient large-scale model serving and personalization. We utilize 500 publicly available LoRAs~\cite{brüelgabrielsson2024compressserveservingthousands} and conduct a continual learning experiment with 50 LoRAs arriving incrementally, updating our \textit{Share} model. We evaluate \textit{Share}’s performance across 2 randomly sampled in-distribution (IID) tasks for each LoRA incremental task batch and 9 out-of-distribution (OOD) tasks (Appendix \cref{tab:LoLA_OOD}) and compare the performance with non-continual model merging method TIES \cite{yadav2024ties}. Here, LoRA $r=16$, \textit{Share} $k=32$, $p=8$ and model is Mistral-7b~\cite{jiang2023mistral7b}. \\
\textbf{Analysis}
\cref{tab:LoLA} demonstrates our method's performance in a strictly continual, zero-shot setting. Results show robust knowledge retention, with most tasks maintaining 90-99\% of original performance across ten sequential learning episodes. \cref{tab:LoLA_OOD} reveals Share's effectiveness on out-of-distribution tasks, achieving an average Rouge-L score of 55.89 across nine OOD tasks—significantly outperforming the non-continual TIES~\cite{yadav2024ties}  while approaching individual LoRAs (73.75). Notably, Share maintains 71-95\% relative performance on most OOD tasks despite using a single set of shared factors rather than task-specific adapters, demonstrating its robust generalization capabilities.
\cref{fig:scale} illustrates Share's efficiency-performance trade-off across both in-distribution and out-of-distribution scenarios. The factor subspace quality improves incrementally with each adapter integration (\cref{fig:icremental-comps}), becoming increasingly representative with diverse knowledge.

A single set of Share factors can compress hundreds of task-specific adapters, yielding \textbf{96× memory savings}. This enables scenarios like large-scale model serving with personalization capabilities. Users can efficiently finetune only lightweight coefficient parameters. \textit{Share} accomplishes this compression through a \textit{data and gradient-free} analytical process, eliminating computational burden typically associated with model merging or continual learning, and facilitating asynchronous, distributed learning paradigms.

\section{Conclusion}
\label{sec:conclusion}
This paper introduces \textit{\textit{Share}}, a novel parametric efficient continual finetuning framework that addresses the challenges of both \textit{efficiently and continually} finetuning large-scale pretrained models. By incrementally learning and updating a shared foundational subspace, \textit{Share} enables the incremental acquisition of new knowledge by a single model without catastrophic forgetting. This approach offers significant advantages in terms of parameter efficiency, memory footprint, and computational cost, and also allows an avenue of reusing community resources (LoRA adapters). Our empirical evaluation demonstrates the effectiveness of \textit{Share} across diverse tasks and large pretrained models, including image classification, natural language understanding, and text-to-image generative models. Notably, we show 7 sets of experiments for computer vision and language tasks, showcasing the comprehensive nature of our evaluation and analysis. 
\textit{Share} could have a positive societal and environmental impact since it enables continual learning for large pretrained models, reduces resource use, and empowers ordinary users and researchers with limited compute resources, to finetune large models continually. Future work aims to extend \textit{Share} to support learning from scratch. 

{
    \small
    \bibliographystyle{ieeenat_fullname}
    \bibliography{main}
}
\clearpage
\setcounter{page}{1}
\maketitlesupplementary

\section{FAQs}
\begin{enumerate}
  \item \textbf{Can Share be initialized without a large number of low-rank adapters?}\\
    Yes. While Share leverages the shared common subspace hypothesis, it can effectively discover and refine this subspace incrementally through continual learning, as demonstrated theoretically in \cref{ssec:theory} and empirically in our experiments. The majority of our strict continual learning evaluations begin with just a single LoRA adapter, from which Share progressively constructs a more comprehensive foundational subspace as additional tasks are encountered.

  \item \textbf{How does Share perform with lower-quality LoRA adapters?}\\
    Share demonstrates robustness to adapter quality variations when task-specific data is available. In such scenarios, Share can leverage its evolving foundational subspace to potentially outperform the original LoRA adapters through cross-task knowledge transfer. However, in purely adapter-based scenarios without accompanying data, the quality of the initial adapters establishes an information bottleneck that constrains Share's performance ceiling for the corresponding tasks.

  \item \textbf{What guided the experimental design across diverse domains?}\\
    Our experimental design prioritizes demonstrating Share's versatility and cross-domain applicability rather than exhaustive evaluation within any single domain. This approach validates Share as a general-purpose solution for parameter-efficient continual fine-tuning across vision, language, and multimodal tasks. We deliberately included challenging scenarios (e.g., text-to-image generation, 3D pose estimation) alongside more standard benchmarks to establish Share's broad utility. The diversity of experiments also underscores Share's practical efficiency, as all evaluations were conducted within reasonable computational constraints.

  \item \textbf{Why benchmark primarily against non-continual LoRA rather than other continual learning methods?}\\
    We establish non-continual LoRA as our primary performance reference because it represents the theoretical upper bound for task-specific adaptation without forgetting constraints. The fact that Share approaches this upper bound while operating under strict continual learning conditions highlights its exceptional effectiveness. Furthermore, the current landscape of parameter-efficient continual learning methods is highly domain-specific—existing approaches typically target individual domains (e.g., image classification or language modeling) and lack the cross-domain versatility that Share provides. Implementing multiple domain-specific baselines would not only increase computational requirements but would contradict the core efficiency principles that motivate our work.

  \item \textbf{Why not integrate conventional continual learning algorithms with LoRA?}\\
    Conventional continual learning algorithms were primarily designed for full-parameter fine-tuning scenarios and may not transfer optimally to the low-rank adaptation paradigm without substantial modifications. The parameter geometry and optimization dynamics differ significantly between full-model and low-rank adaptation settings. Additionally, many established continual learning techniques (e.g., regularization-based approaches) impose computational overheads that would undermine the efficiency benefits of parameter-efficient tuning. While integrating these approaches with LoRA represents an interesting research direction, it falls outside the scope of our current investigation, which focuses on developing a natively parameter-efficient continual learning framework.

\end{enumerate}
\section{Share: Shared Subspace Adaptation}\footnote{We will release Share code which is compatible with HuggingFace PeFT library and a tutorial video here: https://anonymous.4open.science/r/Share-8FF2/}
We provide the detailed algorithm for our method, \textit{Share}, in \cref{algo:share}
\begin{figure}[!htb]
  \centering
  \includegraphics[width=\columnwidth]{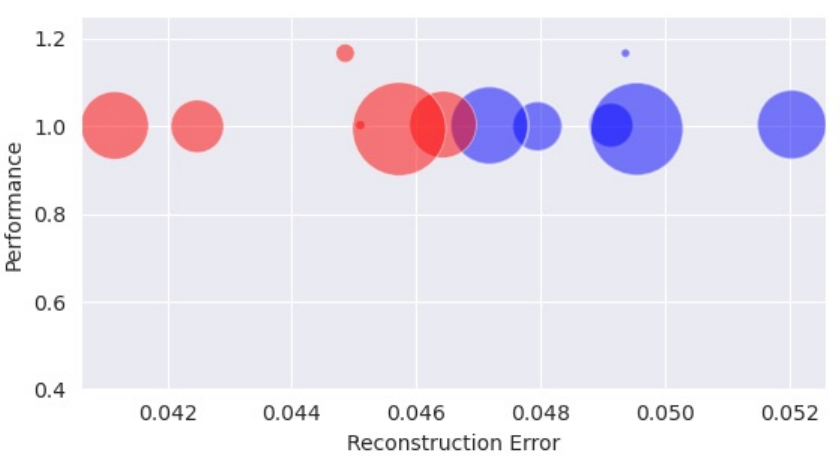}
  \caption{\textit{Low Rank Adapters share a foundational subspace}. We evaluate the \textit{Share}-full model's performance against reconstruction error after finetuning on the GLUE benchmark. Compared with non-continuously trained LoRA submatrices (A and B, shown in red and blue colors), results show that \textit{Share}'s foundational subspace efficiently approximates all LoRAs, suggesting a shared subspace. The radius of the circles represent scaled up standard deviation of the reconstruction error}
  \label{fig:pca}
\end{figure}
\begin{figure}[!htb]
  \centering
  \includegraphics[width=\columnwidth]{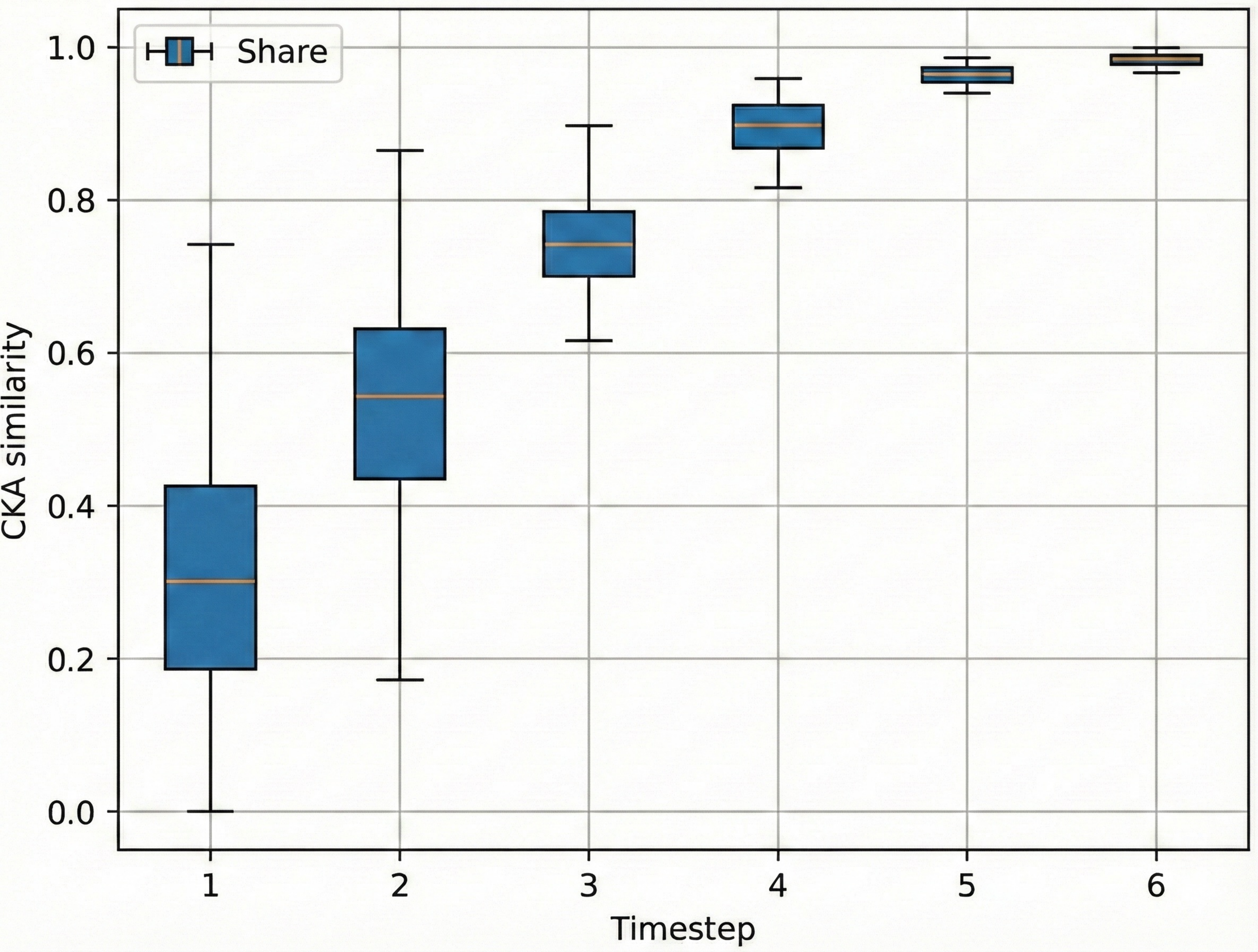} %
  \caption{\textbf{Progression of Factor Subspace of \textit{Share}}. The figure shows CKA similarity \cite{kornblith2019similarityneuralnetworkrepresentations} between \textit{Share}'s final and intermediate factors in the Continual GLUE experiments (\cref{ssec:glue}). \textit{Share} effectively incorporates new factors from incoming data, shown by increased similarity over time (circle size represents variance), while preserving and converging towards optimal principal factors.}
  \label{fig:icremental-comps}
\end{figure}

\begin{algorithm*}[!htb]
  \caption{\textbf{Share: Parameter-Efficient Continual Finetuning via Shared Subspace Adaptation}}
  \label{algo:share}
  \begin{algorithmic}[1]
    \State {\bfseries Input:} LoRA adapters $\{\Delta W^t = (A_t, B_t)\}_{t=1}^T$ where $B_t \in \mathbb{R}^{n\times r}$, $A_t \in \mathbb{R}^{r\times d}$, or task data $\{x^t\}_{t=1}^T$, hyperparameters: principal factors $k$, temporary factors $\varphi$, pseudo-rank $p$
    \State {\bfseries Output:} Principal factors $\alpha^T, \beta^T$, task coefficients $\{\epsilon^t\}_{t=1}^T$ \\

    \State {\bfseries Initialization:} \Comment{Extract foundational subspace}
    \If{$N \geq 1$ LoRA adapters available}
    \State Extract rank vectors: $D_A = [a_{11}, \ldots, a_{Nr}]^\top \in \mathbb{R}^{Nr \times d}$, $D_B = [b_{11}, \ldots, b_{Nr}]^\top \in \mathbb{R}^{Nr \times n}$
    \State Center matrices: $D_A \leftarrow D_A - \bar{D}_A$, $D_B \leftarrow D_B - \bar{D}_B$
    \State Compute SVD: $D_A = U_A\Sigma_A V_A^\top$, $D_B = U_B\Sigma_B V_B^\top$
    \State Extract factors: $\alpha^0 = V_A[:, 1:k] \in \mathbb{R}^{d \times k}$, $\beta^0 = V_B[:, 1:k] \in \mathbb{R}^{n \times k}$
    \Else
    \State Train initial LoRA on first task data and proceed as above
    \EndIf
    \State Initialize coefficients: $\epsilon_\alpha^0, \epsilon_\beta^0 \sim \mathcal{N}(0,\sigma^2) \in \mathbb{R}^{k \times p}$\\

    \For{$t = 1$ to $T$} \Comment{Continual learning process}
    \State $\mathcal{A}^{t-1} = \{\alpha^{t-1}, \beta^{t-1}, \{\epsilon^i\}_{i=1}^{t-1}\}$ \Comment{Current knowledge state}

    \If{receive task data $x^t$}
    \State {\bfseries Continual Adaptation:} \Comment{Temporary expansion of subspace}
    \State $\beta^{t-1\rightarrow t} = \beta^{t-1}[:, 1:\varphi]$, $\alpha^{t-1\rightarrow t} = \alpha^{t-1}[:, 1:\varphi]$ \Comment{Select top factors}
    \State $\epsilon^{t-1\rightarrow t}_{\alpha}, \epsilon^{t-1\rightarrow t}_{\beta} \sim \mathcal{N}(0,\sigma^2) \in \mathbb{R}^{\varphi \times p}$ \Comment{Initialize coefficients}
    \State Forward pass: $h = W_0 x + (\beta^{t-1\rightarrow t}\epsilon_{\beta}^{t-1\rightarrow t})(\alpha^{t-1\rightarrow t}\epsilon_{\alpha}^{t-1\rightarrow t})^\top x$

    \State Optimize $\{\beta^{t-1\rightarrow t}, \alpha^{t-1\rightarrow t}, \epsilon^{t-1\rightarrow t}_{\alpha, \beta}\}$ on task data $x^t$
    \State $\hat{A}_t = ({\alpha^{t-1\rightarrow t}\epsilon_{\alpha}^{t-1\rightarrow t}})^\top $, $\hat{B}_t = \beta^{t-1\rightarrow t}{\epsilon_{\beta}^{t-1\rightarrow t}}$ \Comment{Task adapter}
    \ElsIf{receive LoRA adapter $\Delta W^t$}
    \State $\hat{A}_t = A_t$, $\hat{B}_t = B_t$ \Comment{Direct adapter integration}
    \EndIf

    \State {\bfseries Merging:} \Comment{Knowledge integration}
    \State Reconstruct previous task adapters using current factors:
    \For{$i = 1$ to $t-1$}
    \State $\hat{A}^i = (\alpha^{t-1}{\epsilon_{\alpha}^{i}})^\top$, $\hat{B}^i = \beta^{t-1}{\epsilon_{\beta}^{i}}$
    \EndFor

    \State Construct factor data matrices from all adapters $\{\hat{A}^i, \hat{B}^i\}_{i=1}^t$
    \State $D_A^t = [\text{vec}(\hat{A}^1), \ldots, \text{vec}(\hat{A}_t)]^\top$, $D_B^t = [\text{vec}(\hat{B}^1), \ldots, \text{vec}(\hat{B}_t)]^\top$
    \State Center and perform SVD: $D_A^t = U_A^t\Sigma_A^t {V_A^t}^\top$, $D_B^t = U_B^t\Sigma_B^t {V_B^t}^\top$
    \State Update factors: $\alpha^t = V_A^t[:, 1:k]$, $\beta^t = V_B^t[:, 1:k]$ \Comment{New principal subspace}

    \State {\bfseries Coefficient Recalculation:} \Comment{Analytical knowledge preservation}
    \For{$i = 1$ to $t$}
    \State $\epsilon_{\alpha}^i = {\alpha^{t}}^\top \hat{A}^i$, $\epsilon_{\beta}^i = {\beta^{t}}^\top \hat{B}^i$ \Comment{Project onto new subspace}
    \EndFor

    \State {\bfseries Finetuning:} \Comment{Optional performance enhancement}
    \State Optimize task-specific coefficients $\epsilon^t$ on task data (if available)
    \EndFor
  \end{algorithmic}
\end{algorithm*}

\subsection{Theoretical Analysis}\label{sec:theory_supp}
We restate the notation and proposition from \cref{ssec:theory} here.

\setcounter{nTheorems}{0}
\begin{proposition}
  \label{thm:inc_supp}
  \textbf{(Restating Incremental Subspace Error Bound)}: Let $\mathcal{D}^{t} = [D_1, D_2, \ldots, D_t] \in \mathbb{R}^{N_t \times d}$ be cumulatively stacked weight matrix up to task $\tau_t$, where $N_t = \sum_{i=1}^t n_i$.
  Using the Share approximation at task $t$ with $k$ principal basis vectors we get $\hat{\mathcal{D}}^{t}$ be the rank-$k$ SVD of $\mathcal{D}^{t}$ such that $\hat{\mathcal{D}}^{t}=\text{SVD}(\mathcal{D}^{t})_{[:k]}$.
  Then the reconstruction error satisfies:
  $$\|\mathcal{D}^{t} - \hat{\mathcal{D}}^{t}\|_F^2 = \sum_{i=k+1}^{\min(N_t, d)} (\sigma_i^{(t)})^2,$$ where $\sigma_i$'s are the singular values of $\mathcal{D}^{t}$ for the non-principal basis vectors. See proof in appendix \cref{sec:theory_supp}.

  Furthermore, if tasks are drawn from a distribution with bounded covariance $\mathbb{E}[\|D^{(t)} - \mu\|_F^2] \leq \sigma_{task}^2$, then the expected error growth is bounded.
\end{proposition}

\textbf{Proof:} The result follows directly from the Eckart-Young-Mirsky theorem. For any matrix $M \in \mathbb{R}^{m \times n}$ with SVD $M = \sum_{i=1}^r \sigma_i u_i v_i^T$, the best rank-$k$ approximation in Frobenius norm is $M_k = \sum_{i=1}^k \sigma_i u_i v_i^T$, and:
$$\|M - M_k\|_F^2 = \sum_{i=k+1}^r \sigma_i^2$$

Applying this to $D_{cum}^{(t)}$ gives the stated bound. The distributional bound follows from concentration inequalities for random matrices under the assumption of bounded task covariance.

\begin{definition}[Task Similarity]
  Two tasks $i, j$ are $\delta$-similar if their LoRA parameter distributions satisfy:
  $$\mathbb{E}[\|D^{(i)} - D^{(j)}\|_F^2] \leq \delta \cdot \max(\mathbb{E}[\|D^{(i)}\|_F^2], \mathbb{E}[\|D^{(j)}\|_F^2])$$
\end{definition}

\setcounter{nTheorems}{0}
\begin{corollary}[Bounded Error Growth]
  Under $\delta$-similarity conditions with $\delta < 1$, the Share approximation error grows sub-linearly in the number of tasks:
  $$\mathbb{E}[\|D_{cum}^{(T)} - \hat{D}_{cum}^{(T)}\|_F^2] \leq C_k + \delta \cdot T \cdot \sigma_{task}^2$$
  where $C_k = \sum_{i=k+1}^d \sigma_i^2$ is the truncation error and $\sigma_{task}^2$ bounds task variance.
\end{corollary}

This theorem demonstrates that updating the foundational subspace via Share can approximate the subspace of the full dataset, provided incremental data exhibits similarity to prior data.

\setcounter{nTheorems}{0}
\begin{theorem}\label{theorem:upper_bound_supp} (Restating Theorem \ref{theorem:upper_bound})
  Let $\hegn\in\mathcal{H}_k : \hegn=h(\Wegn,x)$ denote rank-k ERM while finetuning for a new task $\tau_{t+1}$ with true solution $\htrue\in \mathcal{H}:\htrue=h(\Wtrue,x)$ of rank-$n_{t+1}$. Here $\epsilon$ with rank-$k$ is trainable and $\Vk$ are fixed orthonormal basis vectors from $\text{SVD}(\mathcal{D}^t)_{[:k]}$. Given a Lipschitz continuous loss ($\loss(\cdot)$) that is strong convex over the weight spaces spanned by $D$ and $\Vk$ with some Lipschitz constant ($L$), we say with probability at least $1-4\delta$ for some $\delta>0$,
  \begin{equation}
    \label{equation:bound_baseline_supp}
    \Vert \Wtrue - \Wnew\lVert^2_F \leq \sqrt{\frac{C_1n_{t+1}}{\size}}+\sqrt{\frac{2\ln(1/\delta)}{\size}} + C_2
  \end{equation}
  \begin{equation}
    \label{equation:bound_ours_supp}
    \Vert \Wtrue - \Wegn \lVert^2_F \leq C_1\sqrt{\frac{k}{\size}}+\sqrt{\frac{2\ln(1/\delta)}{\size}} + \singularsum + C_2
  \end{equation}
  $\sigma_{k+1}, \dots, \sigma_{n_{t+1}}$ denote singular values of $\Wtrue V_{\perp}$, where $V_{\perp}$ contains right singular vectors of $\mathcal{D}^t$ orthogonal to $V_k$. Here, we assuming $\lVert D \rVert_F \leq B$ and normalized input data.
  \begin{proof}
    The derivation is straightforward.
    We denote by $\mathcal H$ the hypothesis class of linear predictors
    \[
      \mathcal H
      =
      \left\{
        h_D(x)=W_0 x + D x
        \;\middle|\;
        D \in \mathbb R^{n \times d},\ \|D\|_F \le B
      \right\},
    \]
    and by $\mathcal H_k \subset \mathcal H$ the restricted class with rank-constrained parameters,
    \[
      \mathcal H_k
      =
      \left\{
        h_D(x)=W_0 x + D x
        \;\middle|\;
        \operatorname{rank}(D) \le k,\ \|D\|_F \le B
      \right\}.
    \]
    Assuming bounded inputs and an $L$-Lipschitz loss in the prediction, the loss is Lipschitz continuous over both $\mathcal H$ and $\mathcal H_k$. Since $\mathcal H_k$ is non-convex, strong convexity is not assumed globally over $\mathcal H_k$; instead, for both $\mathcal H$ and $\mathcal H_k$, strong convexity is assumed only when optimization is restricted to fixed linear subspaces of parameters (e.g., predictors of the form $D=\varepsilon V^\top$ with fixed $V$), which is sufficient for the analysis and proof of Theorem~2.

    The risk can be written as $\Frisk(h) = E_{\mathcal{S}_t}[\loss(h) ]$ where $\mathcal{S}_t$ is the input data distribution for task $t$. Let us denote $\hegn \in \mathcal{H}_k : \hegn = h(\Wegn,x) = W_0x + \Wegn x$. Similarly, $\htrue \in \mathcal{H} : \htrue = h(\Wtrue,x) = W_0x + \Wtrue x$ and $\hnew \in \mathcal{H} : \hnew = h(\Wnew,x) = W_0x + \Wnew x$. Now, we can write the difference in risks for $\hegn$ and $\htrue$ as
    \begin{align*}
      \Frisk(\hnew) - \Frisk(\htrue) = \mathbb{E}_{\mathcal{S}_t}\big[\loss(\hnew) - \loss(h^*)\big]
    \end{align*}
    By definition of strong convex loss function for some constant $\mu\geq0$,
    \begin{align*}
      \mathbb{E}_{\mathcal{S}_t}\big[\loss(\hnew) - \loss(\htrue)\big] \geq \frac{\mu}{2}\lVert \Wnew - \Wtrue\rVert^2_F
    \end{align*}
    We also know from generalization error bounds using Rademacher Complexity $\Rademacher(\cdot)$ from \citet{bartlett2003rademacher} that with probability at least $1-2\delta$,
    \begin{align*}
      |\Frisk(h) - \emprisk(h)| \leq \frac{\Rademacher(\mathcal{H})}{2} + \sqrt{\frac{\ln(1/\delta)}{2\size}}
    \end{align*}

    Now, we can write risk difference as
    \begin{align*}
      \Frisk(\hnew) - \Frisk(\htrue)  =& \emprisk(\htrue) - \Frisk(\htrue)\\
      &+ \Frisk(h) - \emprisk(\hnew)\\
      &- \emprisk(\htrue) + \emprisk(\hnew)
    \end{align*}
    Since $\hnew$ is the empirical risk minimizer by definition, $\emprisk(\hnew)\leq\emprisk(\htrue)$. Hence,
    \begin{align*}
      \Frisk(\hnew) - \Frisk(\htrue) \leq \emprisk(\htrue) - \Frisk(\htrue)\\
      + \Frisk(\hnew) - \emprisk(\hnew)
    \end{align*}

    Then we take a union bound to conclude that with probability at least $1-4\delta$,
    \begin{align*}
      \Frisk(\hnew) - \Frisk(\htrue) \leq \frac{\Rademacher(\mathcal{H})}{2} + \sqrt{\frac{2\ln(1/\delta)}{\size}}\\
      + \frac{\Rademacher(\mathcal{H})}{2}
    \end{align*}
    Hence, we can also say that with probability at least $1-4\delta$,

    \begin{align}\label{eq:totalerror1}
      \frac{\mu}{2}\lVert\Wtrue - \Wnew\rVert^2_F
      \leq \Rademacher(\mathcal{H}) &+ \sqrt{\frac{2\ln(1/\delta)}{\size}}
    \end{align}

    The Rademacher complexity of a low-rank weight matrix class $\mathcal{H}$ with rank $d$ can be directly bounded using results from \cite{bartlett2003rademacher} as
    \begin{align*}
      \Rademacher(\mathcal{H}) &= \mathcal{O}(\frac{\sqrt{d}\lVert \Wnew \lVert_F}{\sqrt{\size}})\\
      &= \mathcal{O}(\frac{\sqrt{d}}{\sqrt{\size}}) \text{\hspace*{5mm}... for bounded weight norm}
    \end{align*}
    Hence, we get \cref{equation:bound_baseline_supp} as follows:
    \begin{align*}
      \lVert\Wtrue - \Wnew\rVert^2_F
      &\leq C_1\sqrt{\frac{n_{t+1}}{\size}} + \sqrt{\frac{2\ln(1/\delta)}{\size}} + C_2
    \end{align*}
    Similarly, we can also write,
    \begin{align*}
      \frac{\mu}{2}\lVert \Wtruek - \Wegn \rVert^2_F \leq \Frisk(\hegn) - \Frisk(\htruek)
    \end{align*}
    and,
    \begin{align*}
      \Frisk(\hegn) - \Frisk(\htruek)&\leq \frac{\Rademacher(\mathcal{H}_k)}{2} + \sqrt{\frac{2\ln(1/\delta)}{\size}} \nonumber\\&+ \frac{\Rademacher(\mathcal{H}_k)}{2}\\
      &=\Rademacher(\mathcal{H}_k) + \sqrt{\frac{2\ln(1/\delta)}{\size}} \nonumber\\
      &= C_1\cdot\sqrt{\frac{k}{\size}} + \sqrt{\frac{2\ln(1/\delta)}{\size}} + C_2
    \end{align*}
    Hence, together we get,
    \begin{align*}
      \frac{\mu}{2}\lVert \Wtruek - \Wegn \rVert \leq C_1\cdot\sqrt{\frac{k}{\size}} + \sqrt{\frac{2\ln(1/\delta)}{\size}} + C_2
    \end{align*}

    Now to further prove \ref{equation:bound_ours_supp}, we use  properties of Frobenius norm,
    \begin{align*}
      \lVert \Wtrue - \Wegn \rVert^2_F &\leq 2\lVert \Wtrue - \epsilon^* \Vk \rVert^2_F\\
      &+ 2\lVert \Wegn - \epsilon^* \Vk \rVert^2_F\\
      &\leq \singularsum + 2\lVert \Wegn - \epsilon^* \Vk \rVert^2_F\\
      & \text{where $\{\sigma_i\}^d_{k+1}$ are right singular} \\
      & \text{values of $\mathcal{D}^*V_\perp$}\\
      &\leq \singularsum + C_1\cdot\sqrt{\frac{k}{\size}} \\
      &+ \sqrt{\frac{2\ln(1/\delta)}{\size}} + C_2
    \end{align*}
    We use Eckart-Young-Mirsky theorem to get the second inequality above. This gives us \cref{equation:bound_ours_supp} and concludes the proof.

  \end{proof}
\end{theorem}

\section{Experiments}
LoRA~\citep{hu2021lora} implementation for all our experiments is taken from the HuggingFace PEFT~\citep{peft} and Diffusers~\cite{von-platen-etal-2022-diffusers}
library.
\subsection{\textit{Continual} Natural Language Understanding}
The backbone model used is the RoBERTa base~\cite{roberta}.
\paragraph{Hyperparameters}  For LoRA, we use rank$=32$, and for Share, we use $K= 32, \varphi=2, p=8$. Learning rate (lr) $=4e-4$, weight decay$=0.1$, warmup ratio$=0.6$, lr scheduler is Reduce LR on Plateau and batch size $=32$.

\begin{figure}[!h]
  \centering
  \includegraphics[width=\columnwidth]{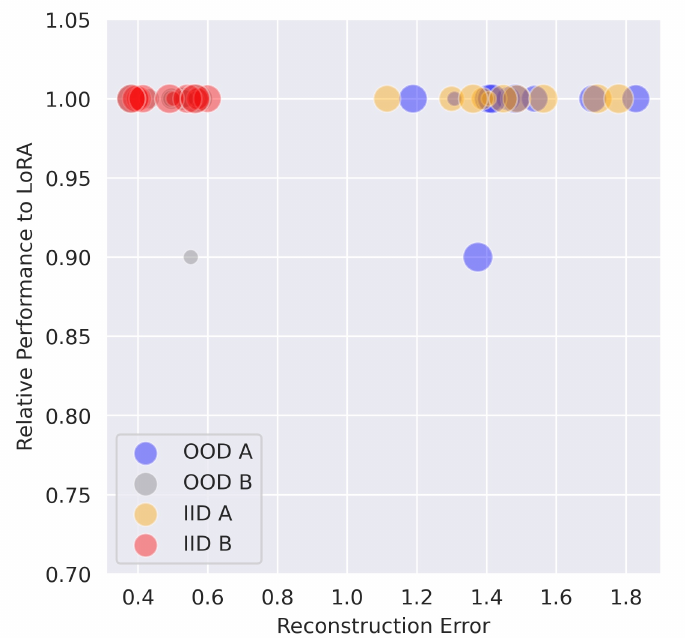}  %
  \caption{\textbf{Continual Model Merging for Hundreds of LoRAs} \textit{Share} demonstrates the capacity to continuously learn and integrate numerous LoRAs (consisting of A and B matrices) alongside the associated task data, as illustrated by this graph of performance against reconstruction error (utilizing the Frobenius Norm).}
  \vspace{-4mm}
  \label{fig:scale}
\end{figure}

\begin{table*}[!h]
  \centering
  \resizebox{\textwidth}{!}{%
    \begin{tabular}{lc|cccccccccccccccccccccccc|c}
      \hline
      & & 769 & 1448 & 362 & 615 & 1168 & 1703 & 065 & 285 & 161 & 1606 & 694 & 1294 & 1479 & 1489 & 296 & 363 & 123 & 129 & 590 & 1453 & 704 & 1400 & Avg\\
      \hline
      \textbf{Upper Bound} & &97.22 & 97.91 & 97.11 & 38.48 & 89.96 & 99.05 & 96.79 & 98.00 & 75.74 & 98.95 & 46.15 & 95.45 & 80.75 & 94.44 & 97.88 & 89.07 & 60.05 & 73.89 & 84.68 & 100 & 72.22 & 9.56 & \textbf{81.52}\\
      \hline
      \multirow{11}{*}{\textbf{Share}}
      &T-0 & 91.36 & 76.04 \\
      &T-1 & 87.55 & 77.29 & 96.17 & 30.51 \\
      &T-2 & 81.63 & 77.29 & 95.70 & 28.87 & 78.70 & 95.21 \\
      &T-3 & 80.16 & 73.89 & 95.94 & 28.55 & 77.16 & 95.08 & 92.31 & 95.07 \\
      &T-4 & 79.26 & 73.89 & 95.94 & 28.24 & 76.85 & 94.88 & 91.72 & 94.77 & 58.09 & 93.68\\
      &T-5 & 78.95 & 73.36 & 95.70 & 28.05 & 75.78 & 94.95 & 91.61 & 94.77 & 58.16 & 93.68 & 46.15 & 45.13\\
      &T-6 & 79.25 & 73.89 & 95.39 & 27.92 & 75.01 & 94.95 & 91.19 & 94.77 & 58.32 & 93.68 & 46.15 & 45.16 & 64.51 & 52.02 \\
      &T-7 & 78.33 & 73.36 & 95.54 & 27.92 & 74.86 & 94.60 & 90.66 & 94.62 & 58.01 & 89.20 & 46.15 & 45.14 & 64.51 & 51.05 & 95.89 & 65.47\\
      &T-8 & 78.38 & 73.36 & 95.54 & 27.81 & 74.55 & 94.95 & 89.62 & 94.62 & 58.01 & 89.20 & 46.15 & 45.19 & 58.01 & 50.85 & 95.89 & 64.51 & 58.08 & 47.59 \\
      &T-9 & 77.31 & 73.36 & 95.39 & 27.91 & 74.39 & 94.95 & 89.58 & 94.62 & 58.16 & 89.20 & 46.15 & 45.09 & 58.01 & 51.77 & 95.89 & 64.08 & 58.11 & 45.61 & 73.15 & 87.04 \\
      &T-10 & 77.31 & 75.14 & 95.39 & 27.64 & 74.39 & 94.95 & 88.42 & 94.62 & 58.32 & 85.81 & 46.15 & 45.09 & 60.51 & 51.77 & 95.89 & 64.38 & 58.12 & 44.77 & 73.32 & 87.04 & 77.77 & 7.64 & \textbf{67.47}\\
      \hline
    \end{tabular}%
  }
  \caption{ Continual Learning with Lots of LoRAs in a zero-shot setting. We report the Rouge-L scores for each of the 2 randomly sampled IID tasks at each time-step.}
  \label{tab:LoLA_full}
\end{table*}

\begin{table*}[!h]
  \centering
  \resizebox{\textwidth}{!}{%
    \begin{tabular}{lccccccccccc}
      \hline
      Method & 039 & 190 & 280 & 290 & 391 & 442 & 1342 & 1391 & 1598 & Avg \\
      \hline
      \textbf{LoRA (Upper Bound)} & 58.77 (1.00) & 86.61 (1.00) & 99.19 (1.00)  & 93.79 (1.00) & 93.45 (1.00) & 67.84 (1.00) & 19.57 (1.00) & 92.92 (1.00) & 51.58 (1.00) & \textbf{73.75} \\
      \textbf{TIES (non-CL)} & 35.18 (0.60) & 22.00 (0.25) & 1.00 (0.01) & 0.00 (0.00) & 78.00 (0.83) & 21.46 (0.32) & 9.93 (0.51) & 1.00 (0.01) & 21.5 (0.42) & \textbf{21.12} \\
      \textbf{Share (CL)} & 46.96 (0.80) & 45.84 (0.53) & 70.56 (0.71) & 84.54 (0.90) & 88.56 (0.95) & 49.78 (0.73) & 10.66 (0.54) & 66.34 (0.71) & 39.81 (0.71) & \textbf{55.89} \\
      \hline
    \end{tabular}%
  }
  \caption{ Continual Learning with Lots of LoRAs. We report the absolute and relative Rouge-L scores for OOD tasks.}
  \label{tab:LoLA_OOD}
\end{table*}

\subsection{\textit{Continual} Image Classification}
\paragraph{Hyperparameters}  For Share, we use $K= 10, \varphi=2, p=1$. We train using
Adam optimizer, learning rate of 0.01, batch size of 128 and 30 epochs. All input images are resized to 224 × 224.

For our baselines, we report the best results from previous works\cite{dap, zhang2024hyperadapter}. Our calculation for trainable parameters is either taken from the specific work, or calculated (as a lower bound). Due to the uncertainty in amount of memory required to save prompts and adapters, we do not report those, but it can be trivially construed that all these methods require significantly more memory than \textit{Share} to save their models.

\subsection{\textit{Continual} Text-to-Image Generation}
We use Flux~\cite{flux} text to image generation model as our base model. We utilize LoRA adapters and data found on HuggingFace community library~\cite{von-platen-etal-2022-diffusers} and merge them continually into our Share model. We follow the recommended hyperparameter recommendations for finetuning Flux models. For LoRA, $r=16$, and for Share $k=32, p=8, \varphi=2$. We provide more samples for our trained Share model. \cref{fig:supp_flux} gives more qualitative examples from our Share model.
\begin{figure*}
  \centering
  \includegraphics[width=0.8\linewidth]{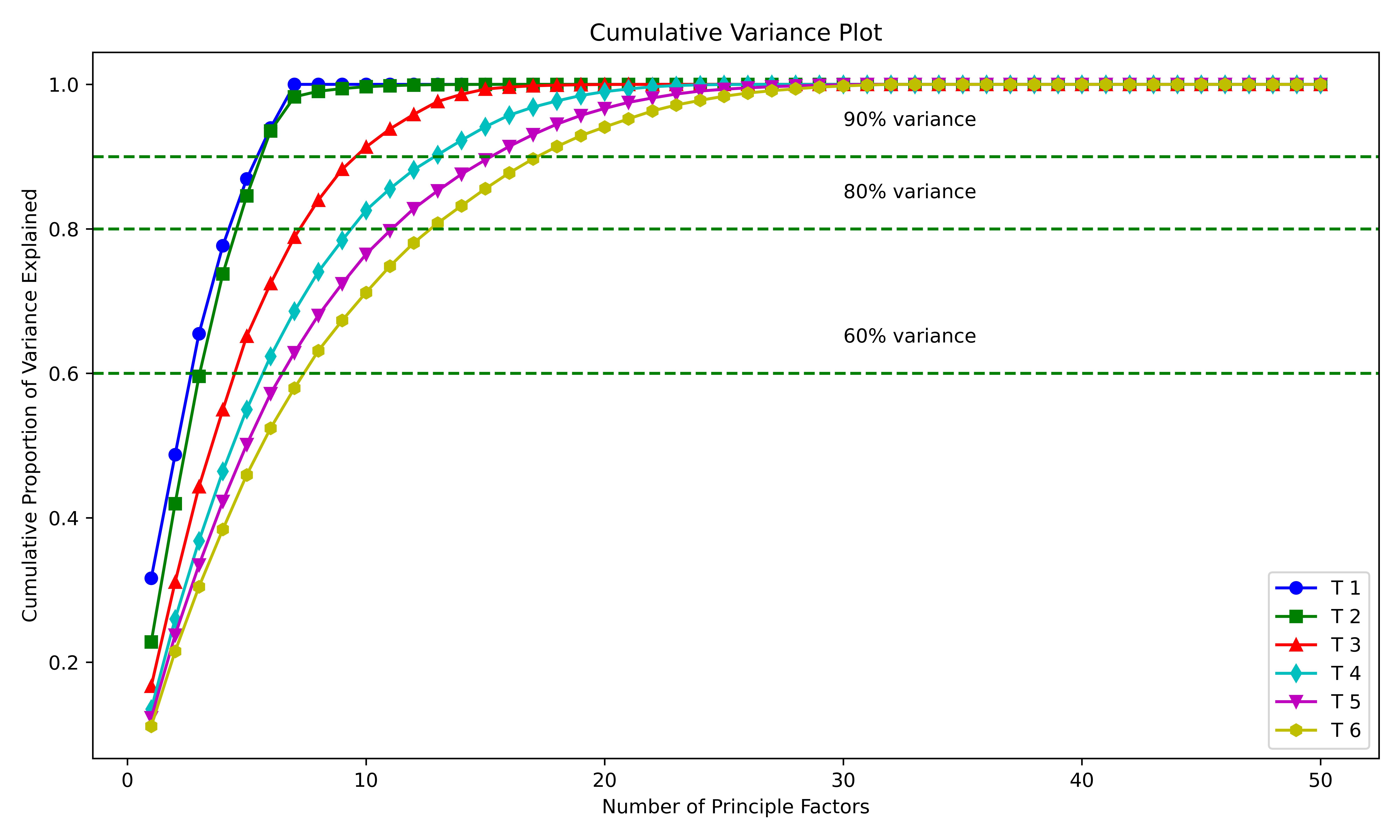}
  \caption{Explained Variance over Time.}
  \label{fig:ablate_vartime}
\end{figure*}

\subsection{Continual Asynchronous Learning and Serving of LoRAs at Scale}
We introduce a novel setup for Continual Model Merging and Learning at Scale, inspired by the continual model merging~\cite{marczak2024magmax} and LoRA scaling~\cite{sheng2023s} literature. This setup enables efficient large-scale model servicing and personalization. We used 500 publicly available LoRAs~\cite{brüelgabrielsson2024compressserveservingthousands} and conducted a continual learning experiment with 50 LoRAs arriving incrementally, updating our Share model. We evaluated Share’s performance across 9 out-of-distribution (OOD) tasks and 2 randomly sampled in-distribution (IID) tasks at each time-step. \cref{tab:LoLA_full} and \cref{tab:LoLA_OOD} show the superior performance of \textit{Share} on both IID and OOD tasks. \textit{Share} outperforms non-CL methods like TIES while nearly matching the performance of LoRA adapters while requiring upto 96x less memory. The hyperparameters are set to LoRA $r=16$, \textit{Share} $k=32$, $p=8$ and model is Mistral-7b~\cite{jiang2023mistral7b}.

\begin{figure}[!h]
  \centering
  \includegraphics[width=\columnwidth]{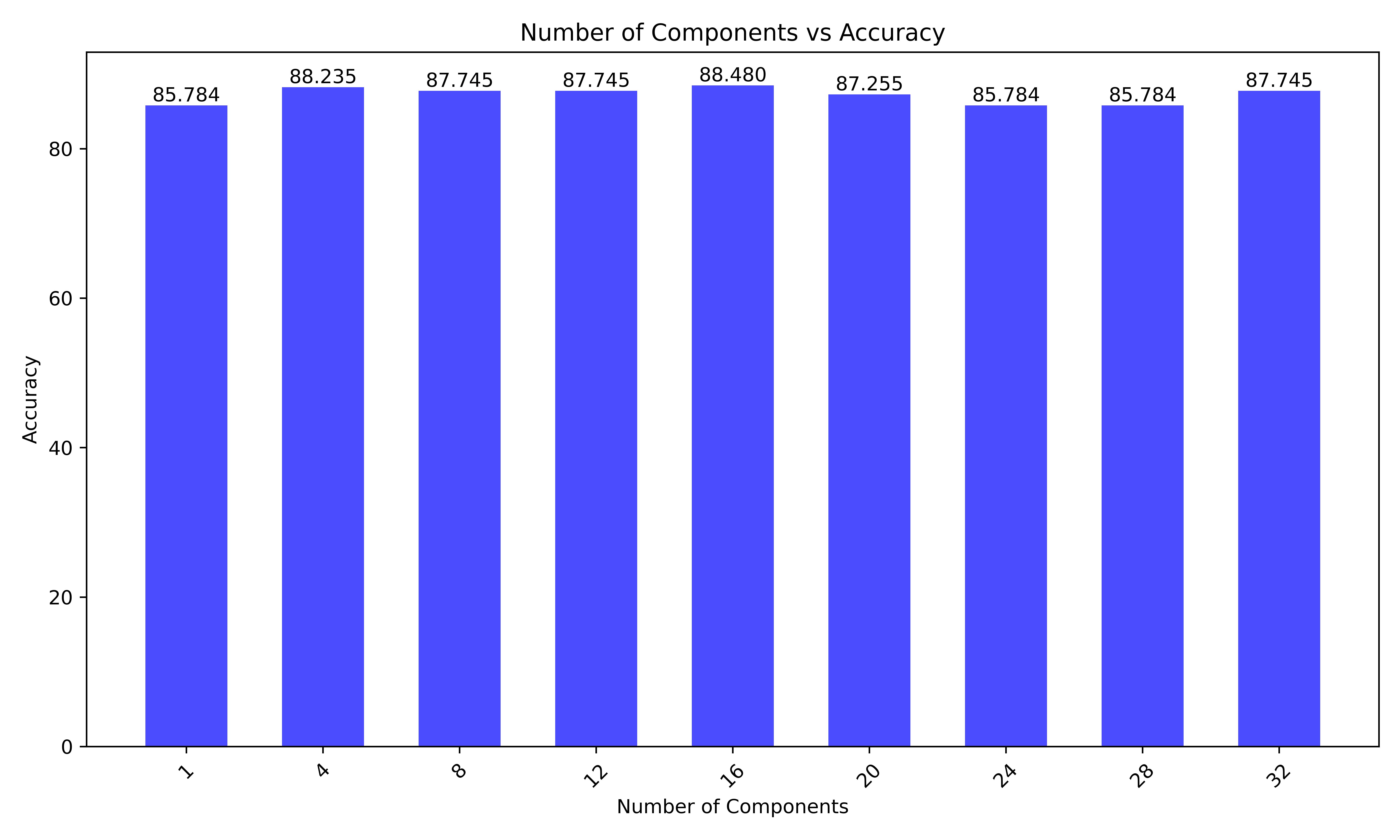}
  \caption{Performance of Share-full on MRPC with different number of principal factors/components}
  \label{fig:k_ablate}
\end{figure}

The idea is to show \textit{Share's} efficacy in being able to compress knowledge (in form of low-rank adapters and data) into a single set of principal factors and importantly, in a continual manner. This has implications in terms of a probable multi-user AI serving model which can cheaply deploy the same Share model to a large number of users while only dynamically switching lightweight coefficient parameters. Importantly, and unlike current models, users would be able to cheaply finetune the coefficients or update their \textit{Share} model and the new temporary factors and coefficients can be efficiently transferred back to the main model server. Using the \textit{Share} algorithm~\cref{algo:share}, the base model can be updated continually and asynchronously, leading to a continuously and efficiently learning AI model.

\begin{figure*}[!h]
  \centering
  \includegraphics[width=0.9\linewidth]{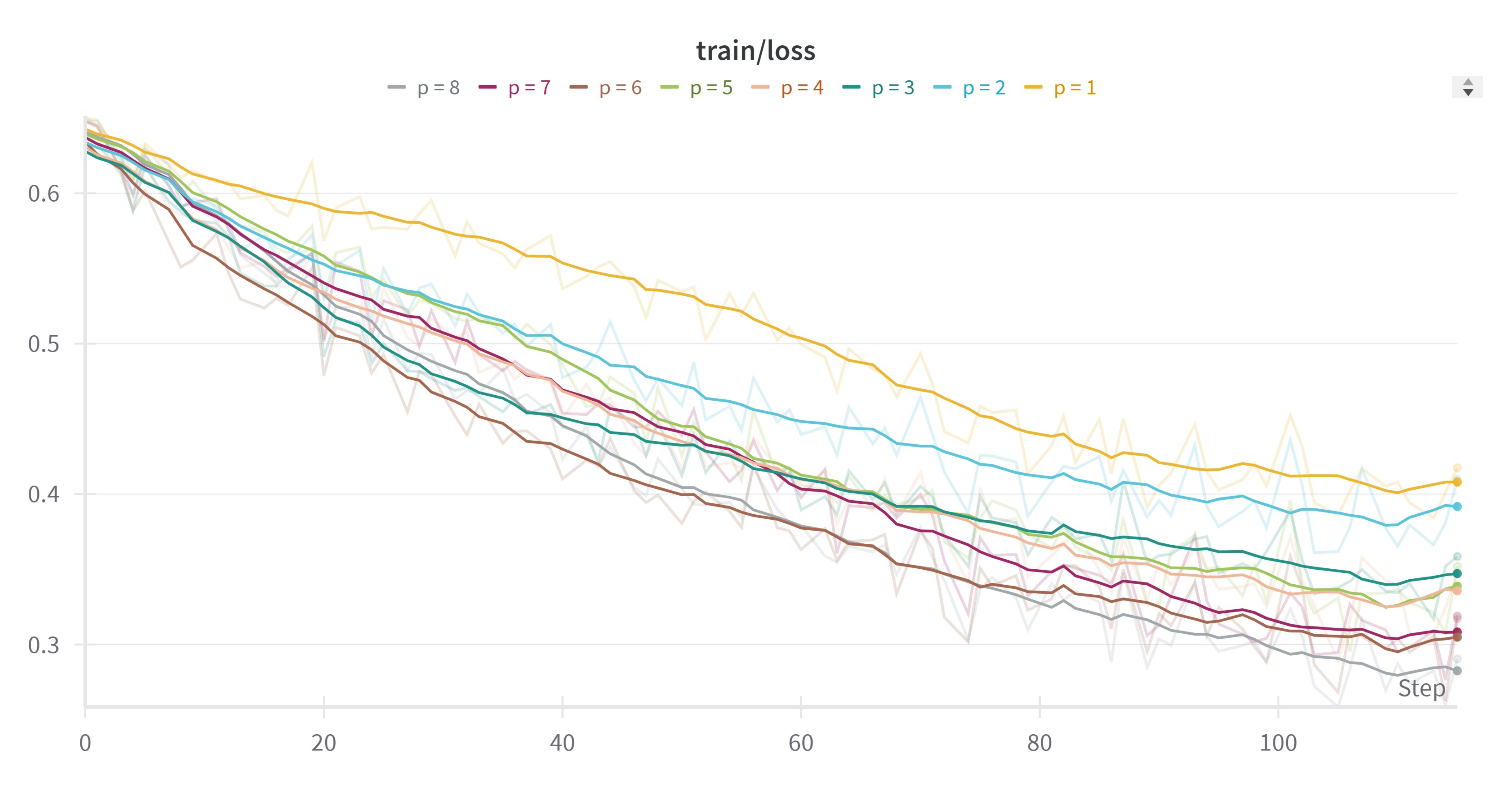}
  \caption{Ablation of Hyperparameter p (dimensionality of coefficients)}
  \label{fig:ablate_p}
\end{figure*}

\subsection{Computational Complexity}

Share's complexity scales favorably with task sequence length:
\begin{itemize}
  \item \textbf{Training:} $O(T \cdot n \cdot d \cdot p)$ vs. LoRA's $O(T \cdot r \cdot d \cdot m)$
  \item \textbf{Storage:} $O(k \cdot (d + m) + T \cdot k \cdot p)$ vs. LoRA's $O(T \cdot r \cdot (d + m))$
\end{itemize}

\subsection{Ablation}\label{sec:ablation_app}
\paragraph{How to decide $k,p, \varphi$?}

The number of Share factors $k$ is determined by a threshold based on the explained variance of the factor data matrix $D$. We find that $k$ with as low as $60\%$ explained variance is effective. Other methods~\cite{pca_bayes, gavishsl} using the eigenvalues of $D$ can also determine $k$. \cref{fig:k_ablate} shows the performance of Share-full finetuned on MRPC with a different number of factors $k$. We note that even $k=2$ works well.

In addition, \cref{fig:ablate_vartime} shows the explained variance over time for the Continual NLU (GLUE) experiment. It can be seen that as the number of continual tasks increases, we need marginally more number of principal factors to encompass most information. However, it can be seen that this number quickly converges around the value $32$.

Our empirical analysis shows that $\varphi=[1,k/4]$ is effective for identifying new factors in the foundational subspace. We show ablative results for different values of $\varphi$ for the Continual Natural Language Understanding (GLUE Benchmark) - RTE task in \cref{tab:ablation_phi}. As can be seen, $\varphi=2$ which we choose in the experiment is a good selection for this hyperparameter.

We also find that a pseudo-rank $p$ of $p=1$ is fairly effective, with higher values yielding additional performance benefits which realtively reduce with increasing $p$ value. Starting with $p=r/3$ is advisable. We ablate this hyperparameter in \cref{fig:ablate_p} for the RTE experiment in the Continual GLUE benchmark and note that $p=1$ works well with the larger improvements noted until $p=4$. The increase in the $p$ value further seems to improve the performance marginally.

\begin{table}[!h]
  \centering
  \begin{tabular}{c|c}
    \toprule
    Value of $\mathbf{\varPhi}$ & Accuracy \\
    \midrule
    $\varphi=1$ & 69.675 \\
    $\varphi=2$ & 73.290\\
    $\varphi=3$ & 74.368 \\
    $\varphi=4$ & 73.285 \\
    \bottomrule
  \end{tabular}
  \caption{Ablation for different values of $\varphi$ for Continual Adaptation part of the method. This was done for the RTE task on the Continual GLUE experiment.}
  \label{tab:ablation_phi}
\end{table}

\subsubsection{Limitations}
Our work represents one of the earliest attempts (if not the first) to enable parameter-efficient continual learning, leveraging either low-rank adapters or data. However, it relies on the assumption of using a single type of backbone model for a given task. Share cannot currently integrate multiple types of pre-trained architectures or models within a single continual learning task. When only adapters are available, the performance of Share on a specific task depends heavily on the quality of the adapter, particularly in the absence of additional data. Furthermore, Share is not yet capable of cross-task continual learning.

\section{Future Work and Broader Impact}
Future work will focus on extending Share to integrate knowledge from diverse, fully trained models. Additionally, enhancing the continual learning framework for large multi-modal models to enable cross-domain and cross-task knowledge transfer is another key direction for future research. Share has the potential for significant positive societal impact, particularly in addressing the environmental and accessibility challenges posed by large vision and language models. By enabling parameter-efficient fine-tuning, it reduces the computational resources required for training, thereby lowering energy consumption and carbon footprints. This makes advanced AI research more accessible to a broader range of researchers, including those with limited access to large-scale computational infrastructure.

Furthermore, Share's ability to support continual learning allows massive datasets and complex tasks to be decomposed into smaller, manageable sub-tasks, enabling federated and distributed learning approaches. This opens the door to more collaborative and decentralized AI development. The memory efficiency and adaptability of Share also pave the way for building lifelong learning systems capable of continuously evolving across diverse tasks and domains. These innovations could lead to more sustainable, inclusive, and scalable AI systems for the future.
\begin{figure*}
  \centering
  \includegraphics[width=\linewidth]{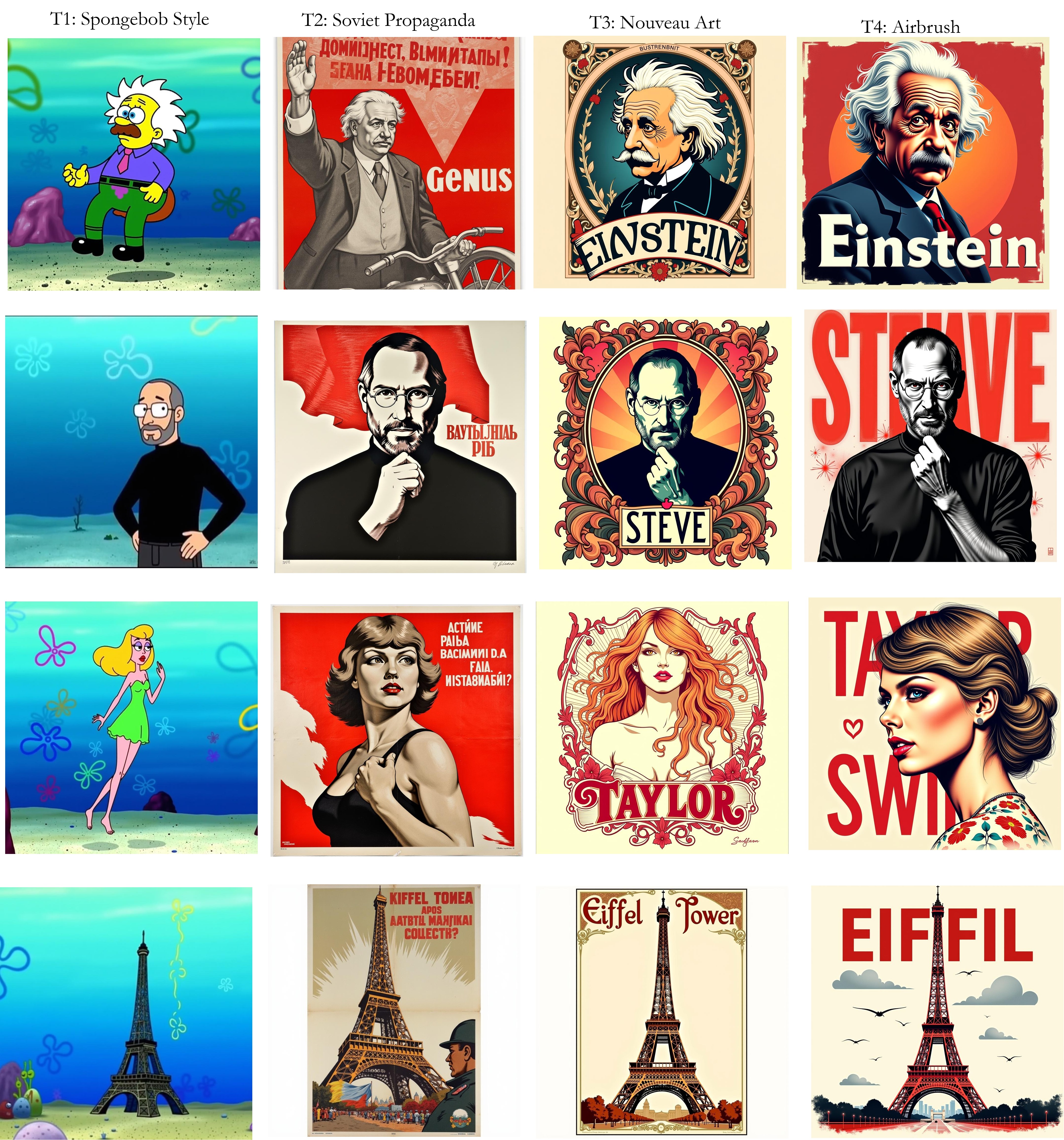}
  \caption{More examples from our text-to-image generation continual experiment using Share}
  \label{fig:supp_flux}
\end{figure*}
\FloatBarrier

\end{document}